\pdfoutput=1
\documentclass{article}
\usepackage[preprint,nonatbib]{neurips_2026}
\usepackage[utf8]{inputenc}
\usepackage[T1]{fontenc}
\usepackage{amsmath,amssymb,amsthm,booktabs,graphicx,microtype}
\usepackage{longtable,enumitem,url,xcolor,hyperref}
\usepackage{caption}
\usepackage{wrapfig}
\usepackage{cleveref}
\newcommand{\R}{\mathbb R}
\newcommand{\sgn}{\operatorname{sign}}
\newtheorem{proposition}{Proposition}

\title{Geometry-Constrained Kolmogorov--Arnold Networks:\\Learning Edge Geometry via Banach Duality}
\author{K S Sesh Kumar\\Brevan Howard Centre for Financial Analysis\\Imperial Business School\\\texttt{skarri@imperial.ac.uk}}
\begin{document}
\maketitle

\begin{abstract}
Kolmogorov--Arnold Networks (KANs) learn univariate functions on edges, making the choice of edge response central to their design. Existing variants use basis expansions such as splines or polynomials, whose functional form is chosen before training. We introduce geometry-constrained KANs, which instead learn the shape of power-law edge responses through an exponent inspired by Banach duality maps. Banach-KAN combines this adaptive power-law response with a bounded hyperbolic tangent (tanh) component. We relate these responses to convex potentials and bound edge sensitivity on bounded input ranges. Across symbolic regression benchmarks, Banach-KAN is competitive with fixed-basis models, particularly with limited clean data. Learned power-law responses remain competitive under measurement noise, while splines are stronger with plentiful clean data. Fixed-exponent and branch-removal controls on a core subset support learning the exponent and retaining both Banach components. Learned exponents describe the shapes of the fitted edge responses.
\end{abstract}

\section{Introduction}
Many scientific regression problems are governed by low-dimensional nonlinear dependencies: pendulum periods depend on amplitude, radiation laws on temperature, and traffic speeds on density. Kolmogorov--Arnold Networks (KANs)~\cite{liu2024kan} are designed for this setting by replacing fixed activations in standard neural networks with learnable univariate edge functions, summing their responses at each node. The central modelling decision in KANs is therefore the choice of edge function.

Existing approaches fix this choice through a basis: B-splines~\cite{liu2024kan}, radial basis functions~\cite{li2024rbfkan}, wavelets~\cite{bozorgasl2024wav}, Chebyshev polynomials~\cite{ss2024chebyshev}, Jacobi polynomials~\cite{aghaei2024fkan}, or Fourier features~\cite{xu2024fourierkan}. Each basis implicitly imposes a function-space geometry---such as smoothness, periodicity, or locality---before any data are observed. Coefficients remain learnable; fractional Jacobi variants also adapt shape and fractional-order parameters~\cite{aghaei2024fkan}.

Figure~\ref{fig:fixed_bases} compares different fixed bases and an $\ell^p$ response on noisy observations of a Heaviside step and a multi-level staircase, with exactly matched trainable parameter counts. Different constructions produce different transition, plateau and oscillation patterns. The examples illustrate how the choice of functional form affects a fitted shape at a common nominal budget; they are motivation for studying adaptive responses, not a superiority claim. Rather than fixing this geometry a priori, we seek a way to adapt it to the data.

\begin{figure}[t]
\centering
\includegraphics[width=\linewidth]{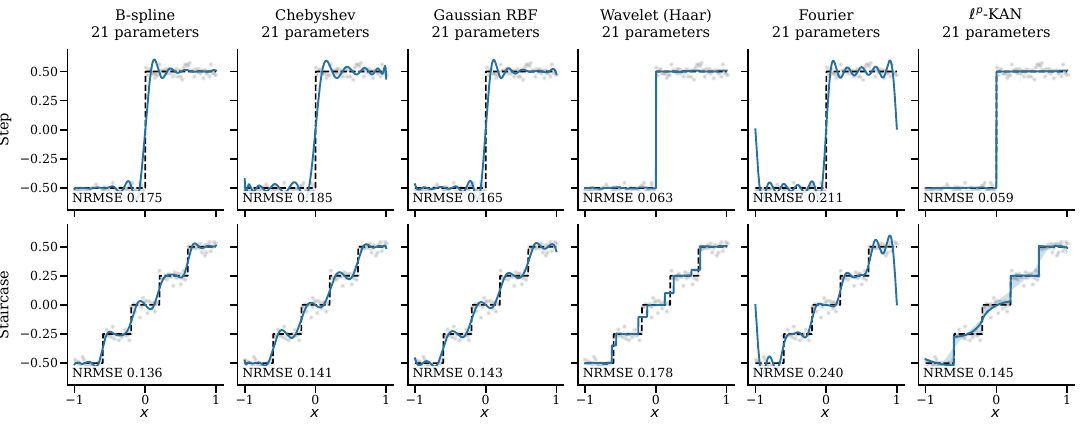}
\caption{Different response families at a matched parameter budget. Each fixed-basis column uses 21 trainable coefficients; the $\ell^p$-KAN column stacks five edges with four trainable parameters each plus one shared offset (21 scalars in total). Noisy step and staircase fits illustrate different fitted behaviours. NRMSE is the median noisy-training error; shading shows the range over five refits.}
\label{fig:fixed_bases}
\end{figure}

The behaviour of a function class---its smoothness, sharpness, and growth---is governed by the geometry of the space it inhabits, which is determined by its underlying norm. A natural way to parameterise such geometries is through a one-parameter family that continuously interpolates between qualitatively different regimes.

This perspective leads to the $\ell^p$ family (\Cref{sec:geometry}), where a single exponent $p$ controls the geometry. Rather than selecting a basis, we parameterise each edge using the duality map associated with this family and learn $p$ directly from data. \emph{Our goal is not to develop a full generalisation theory, but to provide a minimal, interpretable parameterisation of function-space geometry whose empirical effects we study.}

This yields \emph{geometry-constrained KANs}: edge activations whose form is fixed by the underlying geometry but whose geometry is adaptive. The exponent $p$ becomes a simple, interpretable control that moves each edge between sharp, linear, and flatter near-origin regimes.

Learning activation parameters is not new. Slope-adaptive responses~\cite{jagtap2020adaptive}, rational responses~\cite{molina2020pade,boulle2020rational} and learned mixtures~\cite{manessi2018learning} already adapt nonlinear responses during training. Within this established line of work, we study a compact KAN edge family motivated by convex potentials and duality maps, with a learned exponent controlling power-law shape and tail growth. Our experiments compare this family with adaptive-activation alternatives as well as fixed-basis models.

\textbf{Contributions.}
\begin{itemize}[nosep,leftmargin=*]
\item We introduce geometry-constrained KANs, where edge activations are inspired by duality maps with a learned exponent.
\item We provide a unified interpretation of the proposed edge functions as geometry-controlled mappings rather than basis expansions.
\item We compare the proposed family with fixed-basis and adaptive-activation models under noise and data scarcity, and test exponent adaptation and branch combination within the family.
\item We analyse learned exponents as descriptors of fitted edge shapes under a common initialisation.
\end{itemize}

\section{Background and Geometry}
\label{sec:background}
We develop a geometric perspective on KANs in which the choice of edge function is understood as a choice of function-space geometry. This viewpoint allows us to move from fixed basis constructions to a parameterisation in which the geometry itself is learned.

\subsection{Kolmogorov--Arnold Networks}
Kolmogorov--Arnold Networks (KANs) replace fixed nonlinear activations with learnable univariate functions on edges. A layer is defined as follows.
For layer index $\ell$, let $d_{\ell-1}$ and $d_\ell$ denote the input and output widths. For input $x\in\R^{d_{\ell-1}}$, input index $i$ and output index $j$, a KAN layer aggregates its scalar edge responses $\phi^{(\ell)}_{ji}$,
\begin{equation}
 (\Phi_\ell(x))_j=\sum_{i=1}^{d_{\ell-1}}\phi^{(\ell)}_{ji}(x_i).
 \label{eq:kan_layer}
\end{equation}
The model class is determined by the edge functions $\phi_{j,i}$. In contrast to standard neural networks, which fix the activation and learn affine weights, KANs fix the aggregation and learn the nonlinear transformations.

\subsection{Fixed bases as implicit geometry}
Existing KAN variants parameterise $\phi_{j,i}$ using basis expansions, including B-splines~\cite{liu2024kan}, radial basis functions~\cite{li2024rbfkan}, wavelets~\cite{bozorgasl2024wav}, Chebyshev polynomials~\cite{ss2024chebyshev}, Jacobi polynomials~\cite{aghaei2024fkan}, and Fourier features~\cite{xu2024fourierkan}.

Each basis encodes a particular inductive bias---such as smoothness, periodicity, or locality---which determines how functions are represented. The coefficients remain trainable: fixing a basis does not fix the fitted function. A finite basis also allocates capacity through its resolution, so comparisons must account for the chosen budget.

Figure~\ref{fig:fixed_bases} illustrates the consequences. Even with matched parameter budgets, different bases exhibit distinct fitted behaviours on discontinuous targets such as steps and staircases. Smooth bases can approximate discontinuities by distributing changes across coefficients, producing oscillations or bias. This motivates examining the relation between the chosen response family and the structure of the target without establishing a universal ordering of the alternatives.

\begin{figure}[t]
\centering
\includegraphics[width=\linewidth]{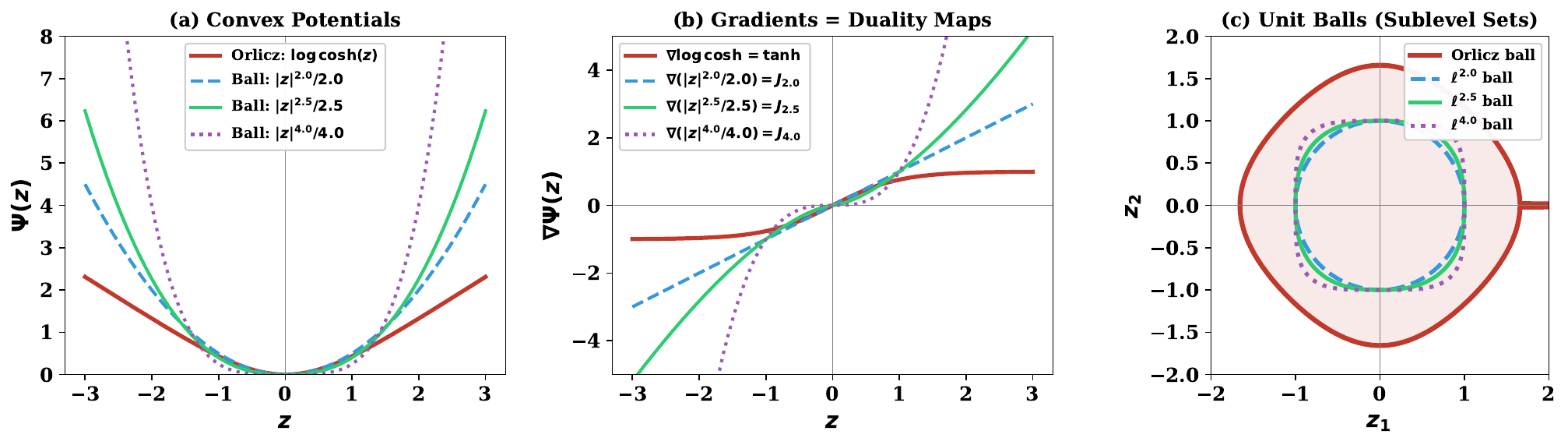}
\caption{Geometry induces activation. Convex potentials define the geometry, their gradients act as activations, and the unit balls illustrate how the geometry varies.}
\label{fig:potentials}
\end{figure}

\subsection{Geometry, duality, and activations}
\label{sec:geometry}
To move beyond fixed bases, we consider geometry at a more fundamental level. In Banach spaces~\cite{megginson1998banach,cioranescu1990geometry}, geometry is induced by a norm: the norm determines the shape of the unit ball and therefore how distances, variations, and gradients are measured. A norm also induces a canonical convex potential that encodes this geometry. For the $\ell^p$ family with $p>1$, the scalar potential and its derivative are
\begin{equation}
 \Psi_p(z)=\frac{|z|^p}{p},\qquad
 J_p(z)=\nabla\Psi_p(z)=\sgn(z)|z|^{p-1},\qquad z\in\R.
 \label{eq:exact}
\end{equation}
For a vector $x\in\R^d$, the corresponding potential is $\|x\|_p^p/p$, whose level sets are scaled $\ell^p$ unit balls. Its gradient gives the power-type duality map~\cite{cioranescu1990geometry,megginson1998banach}, with $J_p$ acting on each coordinate. This converts geometric structure into a nonlinear transformation and establishes the correspondence
\[
 \text{norm}\;\Rightarrow\;\text{geometry}\;\Rightarrow\;\text{convex potential}\;\Rightarrow\;\text{activation}.
\]

Figure~\ref{fig:potentials} visualises this relationship. The convex potentials define the geometry, their gradients define the activations, and the unit balls illustrate how vector geometry varies with $p$. Smaller values of $p$ produce sharper responses, while larger values flatten behaviour near the origin and amplify tails. At $p=2$, the response is linear. The scalar response is zero at the origin for every $p>1$; it is the derivative that becomes singular there when $1<p<2$.

This perspective extends beyond $\ell^p$ geometries. For example, the convex potential
\[
 \Psi_{\tanh}(z)=\log\cosh z
\]
induces the activation
\[
 \nabla\Psi_{\tanh}(z)=\tanh z.
\]
Unlike the power-law growth of $\ell^p$, this potential grows quadratically near the origin and linearly in the tails, leading to bounded, saturating activations. As shown in Figure~\ref{fig:potentials}, these two potentials induce qualitatively different response behaviours. Here $\tanh$ is the canonical potential derivative, not the normalized duality map of an arbitrary Orlicz norm.

Taken together, this yields a unifying view: activation functions arise as gradients of convex potentials determined by geometry. Rather than selecting an activation directly, one selects a potential, and the activation follows. Learning the exponent $p$ therefore adapts the geometry of the associated vector potential and the shape of the scalar edge response. In one dimension this changes the potential and response, rather than creating differently shaped scalar unit balls.

\subsection{A design gap}
The preceding discussion highlights a modelling choice in existing approaches. Basis-based KANs specify a response family through their choice of basis, while standard neural networks typically fix the activation and learn affine weights. The fitted functions adapt during training, but their chosen functional families remain fixed.

Figure~\ref{fig:fixed_bases} illustrates why this choice matters: different response families produce different transitions, plateaus and oscillations on the same targets, even at matched parameter counts. These examples motivate adapting the response shape to the data, although they do not establish geometric mismatch as the sole explanation.

This suggests studying the geometry induced by the response family alongside basis selection. A complementary approach would expose this geometry through parameters that can adapt during learning.

Concretely, we seek a parametrisation that:
\begin{itemize}[nosep,leftmargin=*]
\item retains the simple edge-wise structure of KANs,
\item avoids discrete basis selection or architectural search,
\item exposes geometry through a small number of continuous, interpretable parameters.
\end{itemize}
In the next section, we show that this can be achieved by promoting the exponent $p$ to a learnable parameter on each edge, yielding geometry-constrained KANs.

\section{Geometry-Constrained KANs}
\label{sec:method}
\begin{figure}[t]
\centering
\includegraphics[width=.92\linewidth]{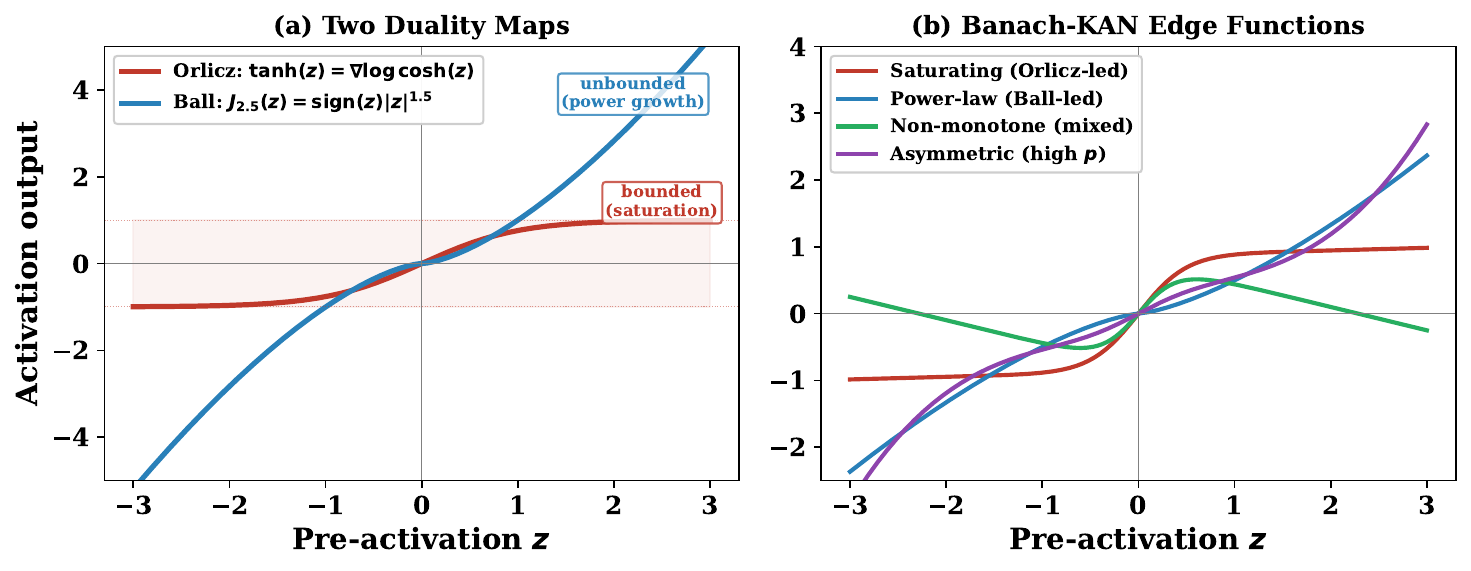}
\caption{\textbf{(a)}~The two duality maps compared: Orlicz duality ($\tanh$) saturates at $\pm 1$, providing bounded expressivity; ball duality ($J_{2.5}$) grows as $|z|^{1.5}$, providing adaptive power-law geometry. \textbf{(b)}~Banach-KAN edge functions span diverse shapes by combining both: saturating curves (Orlicz-led), power-law growth (ball-led), non-monotone functions (mixed), asymmetric responses (high $p$).}
\label{fig:duality}
\end{figure}
We now define three KAN architectures that form a controlled family, each isolating a specific aspect of the geometry--expressivity trade-off (see Figures~\ref{fig:potentials} and~\ref{fig:duality}). In all three cases, the edge function $\phi_{j,i}$ in Eq.~\eqref{eq:kan_layer} takes the form of a parameterized univariate activation with an affine pre-activation $z=wx+b$. For scalar input $x$, $a$ is an output gain, $w$ an input scale, $b$ a bias and $d$ an output offset.

\subsection{Tanh-KAN: Expressivity Without Geometric Adaptation}
Each edge function is
\begin{equation}
 \phi_{\tanh}(x)=a\tanh(wx+b)+d,\label{eq:tanh}
\end{equation}
with $4$ learnable parameters $\{a,w,b,d\}$ per edge.

As discussed in Section~\ref{sec:background}, $\tanh$ is the gradient of a convex potential. It is also one of the most widely studied and deployed activations in practice: the canonical gating nonlinearity in long short-term memory and gated recurrent units~\cite{hochreiter1997lstm,cho2014learning}, theoretically analysed in classical feedforward training studies~\cite{lecun1998efficient}, and a universal approximator for sigmoidal networks~\cite{cybenko1989,hornik1989}. This yields bounded, smooth edge functions with strong expressivity, but offers no \emph{learnable} geometric constraint.

\subsection{$\ell^p$-KAN: Geometric Adaptation}
Each edge function is
\begin{equation}
 \phi_p(x)=aJ_p(wx+b)+d,\label{eq:power}
\end{equation}
with $5$ learnable parameters $\{a,w,b,d,p\}$ per edge. The implemented $J_p$ is the continuously differentiable response in Eq.~\eqref{eq:smooth}. The exponent is parameterised as $p=1+\exp\theta$, with $\theta\in\R$, allowing unconstrained optimisation.

The exponent $p$ governs how each edge responds to its input. Figure~\ref{fig:potentials} illustrates the effect of varying $p$: sharp, threshold-like responses near $p=1$, linear behaviour at $p=2$, and flatter responses near the origin for larger $p$. This control is continuous and learned directly from data.

This makes $p$ a compact descriptor of the edge response: instead of selecting a basis, the model adapts its geometry directly. A fixed-$p$ power edge removes one trainable scalar; learning $p$ does not remove the possible need for regularisation.

\subsection{Banach-KAN: Dual-Geometry Composition}
\label{sec:banachkan}
The two constructions above provide bounded saturation and unbounded power-law growth. Both canonical activations are gradients of convex potentials:
\[
 \tanh=\nabla\Psi_{\tanh},\quad \Psi_{\tanh}(z)=\log\cosh z,\qquad
 J_p=\nabla\Psi_p,\quad \Psi_p(z)=\frac{|z|^p}{p}.
\]
Here $J_p$ is the exact map in Eq.~\eqref{eq:exact}; the implemented response uses the smooth potential in Eq.~\eqref{eq:smooth}.

The Banach-KAN edge combines them:
\begin{equation}
 \phi_B(x)=a_1\tanh(w_1x+b_1)+a_2J_p(w_2x+b_2)+d,\label{eq:banach}
\end{equation}
with $8$ learnable parameters $\{a_1,a_2,w_1,w_2,b_1,b_2,d,p\}$ per edge.

For a shared pre-activation $u=wx+b$ and unit output gains, the canonical branches satisfy
\begin{equation}
 \tanh u+J_p(u)=\nabla(\Psi_{\tanh}+\Psi_p)(u).
\end{equation}
Nonnegative gains similarly give a weighted convex potential. Its conjugate equals the infimal convolution of the components' conjugates because these potentials are finite, continuous and convex~\cite[Ch.~15]{bauschke2017convex}. In practice, independent affine parameters let each branch attend to a different projection of the input, while unrestricted gains allow non-monotone responses (Figure~\ref{fig:duality}). These choices increase expressivity but need not preserve the convex-gradient identity.

Banach-KAN contains Tanh-KAN when $a_2=0$ and therefore inherits universal approximation as width varies~\cite{cybenko1989,hornik1989}. More generally, affine edges and a fixed continuous non-affine response suffice for KAN universality when depth and width may vary~\cite{ismailov2026universality}. These existence statements do not explain finite-budget advantages. Table~\ref{tab:family} lists the family members.

\begin{table}[t]
\centering\footnotesize
\caption{Members of the geometry-constrained KAN family. Parameter lists describe the active scalar response, not a claim of equal capacity or compute across model families.}
\label{tab:family}
\resizebox{\linewidth}{!}{%
\begin{tabular}{@{}llll@{}}
\toprule
Method & Specialisation & Edge form $\phi(x)$ & Active parameters per edge \\
\midrule
\texttt{linear} & $p=2$ in $\ell^p$ & $a(wx+b)+d$ & $(a,w,b,d)$ \\
\texttt{tanh} & $a_2=0$ & $a\tanh(wx+b)+d$ & $(a,w,b,d)$ \\
\texttt{lp\_fixed} & $p$ fixed & $aJ_p(wx+b)+d$ & $(a,w,b,d)$ \\
\texttt{lp} & learned $p$ & $aJ_p(wx+b)+d$ & $(a,w,b,d,\theta)$ \\
\texttt{banach} & combined response & $a_1\tanh(w_1x+b_1)+a_2J_p(w_2x+b_2)+d$ & $(a_1,a_2,w_1,w_2,b_1,b_2,d,\theta)$ \\
\bottomrule
\end{tabular}}
\end{table}

\subsection{Smoothing}
\label{sec:smoothing}
To provide well-defined derivatives throughout optimisation, the proposed $\ell^p$ and Banach models use a continuously differentiable (C$^1$) response, the scalar form of the standard regularisation of the singular power flux~\cite{bartels2018unconditional}. For smoothing parameter $\varepsilon>0$, the implemented response is
\begin{equation}
 J_{p,\mathrm{sm}}^\varepsilon(z)=z(z^2+\varepsilon^2)^{(p-2)/2}
 =\frac{d}{dz}\frac{(z^2+\varepsilon^2)^{p/2}}p.
 \label{eq:smooth}
\end{equation}
For $p>1$ and $\varepsilon>0$ its derivative is positive and finite at every finite input:
\begin{equation}
 (J_{p,\mathrm{sm}}^\varepsilon)'(z)
 =(z^2+\varepsilon^2)^{(p-4)/2}\big[\varepsilon^2+(p-1)z^2\big].
 \label{eq:derivative}
\end{equation}
In particular, the origin slope is $\varepsilon^{p-2}$. Smoothness therefore does not remove potentially large near-origin derivatives when $p<2$.

\subsection{Geometry-dependent sensitivity}
\label{sec:reg}
The exponent $p$ does not only change the shape of the edge function; it also controls how perturbations are amplified through the activation.
For a preactivation bound $M\ge0$ and the implemented smooth map, define $D_p(M,\varepsilon)=\max_{|z|\le M}(J_{p,\mathrm{sm}}^\varepsilon)'(z)$. On an input interval $|x|\le R$, where $R\ge0$, the mean-value theorem gives
\begin{equation}
 \operatorname{Lip}(\phi_p)\le |aw|D_p(|w|R+|b|,\varepsilon).
 \label{eq:lip_lp}
\end{equation}
For Banach-KAN, the $\tanh$ branch is $|a_1||w_1|$-Lipschitz, since $|\tanh'(z)|\le1$. Combining the two branches gives
\begin{equation}
 \operatorname{Lip}(\phi_B)\le |a_1w_1|+|a_2w_2|D_p(|w_2|R+|b_2|,\varepsilon).
 \label{eq:lip}
\end{equation}
At $p=2$, the smooth power response is linear and $D_p(M,\varepsilon)=1$. For $p>2$, the power-law tails make sensitivity depend on the preactivation range.
These elementary bounds concern the smooth response on a bounded domain. They are not guarantees about the noise robustness of a trained model. For the exact map, the derivative is unbounded at zero when $1<p<2$; finite squared-derivative expectation is a different property.

\begin{proposition}[Exact-map Gaussian derivative moment]
\label{prop:moment}
Let $p>1$, let $Z\sim\mathcal N(0,1)$, and fix a preactivation variance $v>0$. For a positive weight scale $\sigma_w$, define
\begin{equation}
 \chi=\sigma_w^2\,\mathbb E[J_p'(\sqrt vZ)^2]
 =\sigma_w^2(p-1)^2v^{p-2}\mathbb E|Z|^{2p-4}.
\end{equation}
Then $\chi$ is finite if and only if $p>3/2$. For a fixed $v$ and $p>3/2$, there is a unique positive $\sigma_w$ satisfying $\chi=1$.
\end{proposition}
\begin{proof}
For $z\ne0$, $J_p'(z)=(p-1)|z|^{p-2}$. A standard Gaussian absolute moment of order $s$ is finite exactly when $s>-1$, and equals $2^{s/2}\Gamma((s+1)/2)/\sqrt\pi$. Substitute $s=2p-4$. Above the threshold the coefficient multiplying $\sigma_w^2$ is finite and positive, proving the second assertion.
\end{proof}
In an idealized Gaussian mean-field stack, $\chi$ describes average squared derivative propagation. Holding $v$ fixed is important: the proposition does not establish a self-consistent variance fixed point, convergence of Adam, or untrainability below the threshold. Finite smoothing makes the corresponding moment finite for every $p>1$. Residual networks and heterogeneous learned exponents also fall outside the scalar calculation. Thus the result explains a regularity distinction, rather than proving that learning geometry improves optimization.

\section{Practical Implementation and Inference}
\label{sec:practical}
We now describe how geometry-constrained KANs are fit in practice. The goal is to estimate the per-edge exponent $\theta$ together with the standard activation parameters from training data.

\textbf{Loss and optimiser.}
Given training pairs $\{(x_i,y_i)\}_{i=1}^n\subset\R^d\times\R$, we minimise the empirical mean squared error
\[
 \mathcal L(\Theta)=\frac1n\sum_{i=1}^n(f_\Theta(x_i)-y_i)^2,
\]
with $\Theta$ collecting the per-edge parameters. For the C$^1$ responses, optimisation is unconstrained on $\theta$, with $p=1+\exp\theta$ enforcing $p>1$. We use Adam~\cite{kingma2014adam} with cosine annealing and gradient clipping. Checkpoints are selected on a validation set carved from the training pool, as specified in Section~\ref{sec:experiments}.

\label{sec:implementation}
The experiments use the C$^1$ response defined in Section~\ref{sec:smoothing}, with $\varepsilon=10^{-8}$ and a learned exponent. The $\ell^p$ (additive $\varepsilon$) comparator uses the response
\begin{equation}
 J_{p,\mathrm{add}}^\varepsilon(z)=\sgn(z)(|z|+\varepsilon)^{p-1},\qquad \varepsilon=10^{-8}.
 \label{eq:additive}
\end{equation}
For this comparator, $p-1=\operatorname{softplus}(\theta)=\log(1+\exp\theta)$, whereas the C$^1$ models use $p-1=\exp\theta$. For every $p>1$ this map has a jump of size $2\varepsilon^{p-1}$ at zero. It is not a continuous smoothing, and its finite derivative away from zero does not give a Lipschitz certificate across the origin. The sensitivity bounds in Section~\ref{sec:reg} apply to the C$^1$ responses in Eq.~\eqref{eq:smooth}, not to this comparator. Results for the two numerical maps are distinguished explicitly; mathematical continuity alone does not establish empirical equivalence.

\section{Experiments}
\newcommand{\CleanSplineMedian}{0.019}
\newcommand{\CleanSplineRank}{1.82}
\newcommand{\CleanSplineWins}{32}
\newcommand{\ModerateSplineMedian}{0.163}
\newcommand{\HeavySplineMedian}{0.743}
\newcommand{\ScarceSplineMedian}{0.213}
\newcommand{\CleanBanachMedian}{0.032}
\newcommand{\CleanBanachRank}{5.16}
\newcommand{\CleanBanachWins}{4}
\newcommand{\ModerateBanachMedian}{0.114}
\newcommand{\HeavyBanachMedian}{0.349}
\newcommand{\ScarceBanachMedian}{0.138}
\newcommand{\CleanLpMedian}{0.069}
\newcommand{\CleanLpRank}{11.10}
\newcommand{\CleanLpWins}{0}
\newcommand{\ModerateLpMedian}{0.116}
\newcommand{\HeavyLpMedian}{0.283}
\newcommand{\ScarceLpMedian}{0.170}
\newcommand{\CleanAdditiveMedian}{0.076}
\newcommand{\CleanAdditiveRank}{11.98}
\newcommand{\CleanAdditiveWins}{0}
\newcommand{\ModerateAdditiveMedian}{0.126}
\newcommand{\HeavyAdditiveMedian}{0.275}
\newcommand{\ScarceAdditiveMedian}{0.171}
\newcommand{\CleanTanhRatio}{0.50}
\newcommand{\CleanSlopeRatio}{0.44}
\newcommand{\CleanAdditiveRatio}{0.49}
\newcommand{\HeavyPowerRatio}{1.03}
\newcommand{\ScarceSplineRatio}{0.73}
\newcommand{\LargeSplineRatio}{2.01}
\newcommand{\ScarceBanachRank}{4.1}
\newcommand{\CoreLearnedLp}{0.127}
\newcommand{\CoreFixedLp}{0.203}
\newcommand{\CoreLearnedBanach}{0.113}
\newcommand{\CoreFixedBanach}{0.140}
\newcommand{\CoreCleanBanach}{0.031}
\newcommand{\CorePowerOff}{0.082}
\newcommand{\CoreSpline}{0.037}
\newcommand{\CoreSplineAdam}{0.042}
\newcommand{\CoreSplineLong}{0.042}
\newcommand{\LpExponentMedian}{2.30}
\newcommand{\LpBelowTwoPercent}{28\%}
\newcommand{\BanachExponentMedian}{2.35}
\newcommand{\BanachBelowTwoPercent}{19\%}
\newcommand{\ExtrapBanachMedian}{0.416}
\newcommand{\ExtrapSplineMedian}{0.582}
\newcommand{\ExtrapTanhMedian}{1.031}
\newcommand{\ExtrapSplineRatio}{0.86}
\newcommand{\ExtrapTanhRatio}{0.57}
\newcommand{\CoreHeavySpline}{0.784}
\newcommand{\CoreHeavyCV}{0.329}
\newcommand{\CoreHeavyLp}{0.287}
\newcommand{\CoreScarceLp}{0.386}
\newcommand{\CoreScarceCV}{0.427}
\newcommand{\CoreScarceBanach}{0.393}
\newcommand{\TiedCleanRatio}{1.269}
\newcommand{\TiedCleanInterval}{[1.143, 1.396]}
\newcommand{\TiedNoiseRatio}{0.943}
\newcommand{\TiedNoiseInterval}{[0.889, 1.005]}

\label{sec:experiments}
In this section we evaluate geometry-constrained KANs on symbolic-regression targets, comparing clean fitting, robustness to measurement noise and small-sample approximation, with Penn Machine Learning Benchmarks (PMLB) tabular regression~\cite{romano2022pmlb} as an additional check.

\textbf{Datasets.}
The benchmark contains 50 equations in two to six variables, drawn from the AI Feynman collection~\cite{udrescu2020ai} and synthetic stress tests. The first 18 implemented equations form the core subset. \emph{The main results for clean fitting, measurement noise and sample scarcity use all 50 equations. Additional fixed-exponent, branch-removal and optimisation controls use only the 18-equation core subset.} Extrapolation also uses all 50 equations. Appendix~\ref{app:equations} lists every implemented formula.

\textbf{Protocol.}
For each equation we generate a training pool and 1,000 separate clean test examples. Targets are standardized using training-pool statistics, and independent Gaussian noise of standard deviation $\sigma$ is added to standardized training labels. Adam models reserve 20\% of the pool for validation and select checkpoints by validation loss every 500 steps and at the final step; test labels do not influence fitting or selection. Training-pool standardization includes validation labels. All main comparisons use 5,000 Adam steps, initial learning rate $0.01$, cosine decay and gradient-norm clipping at 10. KAN widths are $[d,K,1]$, with $K=5$ for $d\le2$, $K=4$ for $d=3,4$ and $K=3$ otherwise; the multilayer perceptron (MLP) has one width-50 hidden layer for $d\le3$ and two otherwise. Learned exponents start at $2.5$. Complete configurations are supplied with the code.

For symbolic regression we report normalized root mean squared error (NRMSE): test root mean squared error divided by the test-target sample standard deviation. Each equation's score is the median across five runs; benchmark NRMSE is the median of those equation scores. Mean ranks and wins compare the displayed model roster, with exact ties sharing wins. Equation-paired error ratios and bootstrap intervals are reported separately in Appendix~\ref{app:paired}; they need not have the same ordering as aggregate medians.

\textbf{Baselines.}
The 15-model comparison includes the three proposed responses; the $\ell^p$ (additive $\varepsilon$) variant in Eq.~\eqref{eq:additive}, whose results we also report; per-edge adaptive-slope $\tanh$~\cite{jagtap2020adaptive}, safe rational/Pad\'e~\cite{molina2020pade} and unconstrained activation-mixture~\cite{manessi2018learning} controls; B-spline, Chebyshev~\cite{ss2024chebyshev}, radial-basis~\cite{li2024rbfkan}, Fourier~\cite{xu2024fourierkan}, wavelet~\cite{bozorgasl2024wav} and fixed-parameter Jacobi (Legendre) edges; a Lipschitz-regularised spline; and the MLP. Both spline rows use cubic B-splines with grid size $G=3$, trained by Adam; the regularised row fixes its derivative penalty at $10^{-2}$. Parameters are reported in Appendix~\ref{app:parameters}. These are natural model budgets, not an equal-parameter comparison across every family. The separate core controls retain their specified spline optimisation and cross-validation protocols. Comparator definitions are detailed in Appendix~\ref{app:sweep}.

\begin{table}[!t]
\centering\small
\setlength{\tabcolsep}{2pt}
\caption{Main results on all 50 equations, comparing 15 models under the shared Adam protocol. Entries are median test NRMSE. Clean fitting uses $n=500$; wins and mean ranks compare all 15 models. Noise varies at $n=500$; scarcity uses clean labels. Bold marks lower error and rank or higher wins. Complete ranks are in Appendix~\ref{app:matched_details}.}
\label{tab:matched}
\begin{tabular}{lrrrrrrrrrrr}
\toprule
& \multicolumn{3}{c}{Clean} & \multicolumn{4}{c}{Label noise $\sigma$ ($n=500$)} & \multicolumn{4}{c}{Clean training size $n$} \\
\cmidrule(lr){2-4}\cmidrule(lr){5-8}\cmidrule(lr){9-12}
Model & NRMSE & Wins & Rank & 0.1 & 0.3 & 0.5 & 1.0 & 50 & 100 & 200 & 1000 \\
\midrule
Banach & 0.032 & 4 & 5.16 & 0.049 & 0.114 & 0.185 & 0.349 & 0.138 & \textbf{0.069} & 0.041 & 0.029 \\
$\ell^p$ & 0.069 & 0 & 11.10 & 0.076 & 0.116 & \textbf{0.162} & 0.283 & 0.170 & 0.107 & 0.075 & 0.058 \\
Tanh & 0.068 & 0 & 12.30 & 0.073 & 0.118 & 0.178 & 0.324 & 0.211 & 0.141 & 0.101 & 0.066 \\
$\ell^p$ (additive $\varepsilon$) & 0.076 & 0 & 11.98 & 0.082 & 0.126 & 0.167 & \textbf{0.275} & 0.171 & 0.120 & 0.095 & 0.066 \\
Adaptive-slope Tanh & 0.070 & 0 & 12.94 & 0.080 & 0.133 & 0.194 & 0.354 & 0.257 & 0.163 & 0.114 & 0.072 \\
Activation mixture & 0.042 & 1 & 8.36 & 0.061 & 0.120 & 0.180 & 0.340 & 0.163 & 0.103 & 0.057 & 0.038 \\
Rational/Pad\'e & 0.029 & 7 & 4.48 & \textbf{0.043} & \textbf{0.111} & 0.174 & 0.334 & \textbf{0.126} & 0.071 & 0.042 & 0.021 \\
Spline (Adam) & \textbf{0.019} & \textbf{32} & \textbf{1.82} & 0.057 & 0.163 & 0.295 & 0.743 & 0.213 & 0.083 & \textbf{0.030} & \textbf{0.011} \\
Spline (Adam, Lip.-reg.) & 0.063 & 2 & 9.98 & 0.075 & 0.156 & 0.259 & 0.545 & 0.297 & 0.166 & 0.088 & 0.044 \\
Chebyshev & 0.033 & 0 & 5.90 & 0.053 & 0.145 & 0.234 & 0.542 & 0.236 & 0.093 & 0.045 & 0.026 \\
RBF & 0.028 & 0 & 4.82 & 0.053 & 0.133 & 0.207 & 0.501 & 0.206 & 0.096 & 0.039 & 0.021 \\
Fourier & 0.030 & 2 & 6.48 & 0.057 & 0.165 & 0.279 & 0.622 & 0.509 & 0.151 & 0.049 & 0.023 \\
Wavelet & 0.036 & 0 & 8.96 & 0.061 & 0.165 & 0.272 & 0.557 & 0.367 & 0.166 & 0.067 & 0.028 \\
Jacobi & 0.036 & 0 & 6.22 & 0.055 & 0.146 & 0.238 & 0.552 & 0.233 & 0.101 & 0.050 & 0.024 \\
MLP & 0.057 & 2 & 9.50 & 0.076 & 0.166 & 0.258 & 0.492 & 0.219 & 0.127 & 0.082 & 0.037 \\
\bottomrule
\end{tabular}

\end{table}

\subsection{Benchmark equations}
\label{sec:matched}
\textbf{Clean fitting.}
Table~\ref{tab:matched} shows that the Adam-trained spline has the lowest aggregate clean error and mean rank: median NRMSE \CleanSplineMedian, \CleanSplineWins{} wins and mean rank \CleanSplineRank. Banach reaches median NRMSE \CleanBanachMedian, \CleanBanachWins{} wins and mean rank \CleanBanachRank, while the rational, radial-basis and Fourier edges have lower aggregate medians. Banach does not lead this comparison. Against several other trainable responses, Banach has lower equation-paired error: its geometric-mean ratios against $\tanh$, adaptive-slope $\tanh$ and the $\ell^p$ (additive $\varepsilon$) response are \CleanTanhRatio, \CleanSlopeRatio{} and \CleanAdditiveRatio, respectively, with family-adjusted intervals below one (Appendix~\ref{app:paired}).

\textbf{Noisy data.}
The rational control has the smallest reported median at $\sigma=0.1$ and $0.3$, $\ell^p$ at $\sigma=0.5$, and $\ell^p$ (additive $\varepsilon$) at $\sigma=1$. Close medians do not by themselves establish a paired advantage: at $\sigma=0.3$, $\ell^p$ and Banach give \ModerateLpMedian{} and \ModerateBanachMedian, respectively. At $\sigma=1$, $\ell^p$ reaches \HeavyLpMedian, compared with \HeavyAdditiveMedian{} for $\ell^p$ (additive $\varepsilon$) and \HeavySplineMedian{} for the Adam spline. The paired C$^1$/additive-$\varepsilon$ power-response ratio is \HeavyPowerRatio{} with an interval including one, so these experiments do not distinguish the two tested power-law implementations under heavy noise. The results support different operating regimes rather than uniform superiority of one response family.

\textbf{Small-sample approximation.}
At $n=50$, Banach reaches \ScarceBanachMedian{} NRMSE; rational has the lowest median and mean rank. Banach's paired ratio against the Adam spline is \ScarceSplineRatio, with a family-adjusted interval below one. At $n=100$, Banach has the lowest median; rational has the best rank. The spline has the lowest median at $n=200$ and $1{,}000$; at $n=1{,}000$, the Banach/spline ratio is \LargeSplineRatio. Banach's advantage over the Adam spline is concentrated in scarce-data conditions.

\subsection{Architecture dissection on the core subset}
\label{sec:dissection}
We use only the 18-equation core for the fixed-exponent, branch-removal and additional optimiser controls in Table~\ref{tab:ablation}. These controls are separate from the 50-equation comparison above.
\begin{table}[t]
\centering\small
\begin{minipage}{.32\linewidth}\centering\footnotesize\setlength{\tabcolsep}{1.5pt}\begin{tabular}{lr}
\toprule
Model & Median NRMSE \\
\midrule
$\ell^p$, fixed $p=1.5$ & 0.172 \\
$\ell^p$, learned $p$ & 0.088 \\
Banach, fixed $p=2$ & 0.063 \\
Banach, learned $p$ & \textbf{0.031} \\
\bottomrule
\end{tabular}
\\[2pt]{\footnotesize (a) clean}
\end{minipage}\hfill
\begin{minipage}{.32\linewidth}\centering\footnotesize\setlength{\tabcolsep}{1.5pt}\begin{tabular}{lr}
\toprule
Model & Median NRMSE \\
\midrule
$\ell^p$, fixed $p=2.5$ & 0.203 \\
$\ell^p$, learned $p$ & 0.127 \\
Banach, fixed $p=2$ & 0.140 \\
Banach, learned $p$ & \textbf{0.113} \\
\bottomrule
\end{tabular}
\\[2pt]{\footnotesize (b) $\sigma=0.3$}
\end{minipage}\hfill
\begin{minipage}{.32\linewidth}\centering\footnotesize\setlength{\tabcolsep}{2pt}\begin{tabular}{lr}
\toprule
Model & Clean NRMSE \\
\midrule
Banach & \textbf{0.031} \\
Banach, power off & 0.082 \\
Spline G=3 & 0.037 \\
Spline Adam & 0.042 \\
\bottomrule
\end{tabular}
\\[2pt]{\footnotesize (c) clean}
\end{minipage}
\caption{Additional controls on the 18-equation core only, with training-pool size 500. (a) Clean and (b) noisy fixed-exponent controls show the lowest-median fixed $\ell^p$ exponent in each condition, selected descriptively from the full sweep. (c) Branch removal and optimiser controls. Full tables are in Appendix~\ref{app:controls}.}
\label{tab:ablation}
\end{table}
Learning the exponent gives the power-only model lower median error than the best fixed exponent in each condition (Table~\ref{tab:ablation}(a,b)). At $\sigma=0.3$, learned $\ell^p$ reaches \CoreLearnedLp, compared with \CoreFixedLp{} at fixed $p=2.5$; learned Banach reaches \CoreLearnedBanach{} versus \CoreFixedBanach{} at fixed $p=2$. Paired tests retain the declared fixed-$p=2.5$ reference, rather than selecting it anew from the displayed minima. These controls support adaptation within this response family, not an advantage attributable solely to a geometric interpretation.

Removing the Banach power branch (Table~\ref{tab:ablation}(c)) raises clean median error from \CoreCleanBanach{} to \CorePowerOff. The intact family therefore benefits from the branch under this protocol, although disabling it also changes effective capacity. The spline optimiser control gives clean medians of \CoreSpline{} with fixed-step L-BFGS, \CoreSplineAdam{} with Adam and \CoreSplineLong{} with 10,000 Adam steps. These controls do not establish invariance to optimiser or regularisation choices. Appendix~\ref{app:controls} gives their scope and selection rules.

\subsection{Learned geometry and extrapolation}
\label{sec:exponents}
\begin{table}[t]
\centering\small
\caption{Learned geometry, extrapolation and affine controls. (a) Clean exponent summaries from the main fits; (b) extrapolation, both on 50 equations. (c) Tied-affine median NRMSE on the 18-equation core, $n=500$. Full protocols and results: Appendices~\ref{app:controls} and~\ref{app:extrapolation}.}
\label{tab:geometry}\label{tab:extrap}
\begin{minipage}[t]{.56\linewidth}\vspace{0pt}\centering
\begin{tabular}{lrr}
\toprule
Metric & $\ell^p$ & Banach \\
\midrule
Median $p$ & 2.296 & 2.350 \\
5th percentile & 1.445 & 1.646 \\
95th percentile & 4.439 & 4.405 \\
Fraction $p<2$ & 0.284 & 0.187 \\
\bottomrule\end{tabular}\par\smallskip (a)\end{minipage}\hfill
\begin{minipage}[t]{.38\linewidth}\vspace{0pt}\centering
\begin{tabular}{lr}
\toprule
Model & Extrap. NRMSE \\
\midrule
Banach & \textbf{0.416} \\
$\ell^p$ & 0.664 \\
$\tanh$ & 1.031 \\
Spline (Adam) & 0.582 \\
\bottomrule\end{tabular}\par\smallskip (b)\end{minipage}\par\medskip
\begin{tabular}{lrrr}
\toprule
Affine variant & $\sigma=0$ & $\sigma=0.3$ & $\sigma=1$ \\
\midrule
Independent & \textbf{0.031} & \textbf{0.113} & 0.323 \\
Tied, matched gain & 0.037 & 0.117 & 0.313 \\
Tied & 0.037 & 0.117 & 0.320 \\
Tied, $p_0=2$ & 0.059 & 0.120 & \textbf{0.308} \\
\bottomrule\end{tabular}
\par\smallskip (c)

\end{table}
The learned exponents move away from their common initial value (Table~\ref{tab:geometry}(a)). Their pooled medians are \LpExponentMedian{} for $\ell^p$ and \BanachExponentMedian{} for Banach, with about \LpBelowTwoPercent{} and \BanachBelowTwoPercent{} of edges below $p=2$, respectively. The fits therefore use both steep near-origin responses and flatter responses with faster-growing tails. These are descriptive, edge-weighted distributions; they do not make the exponent an identifiable property of the target or establish stability across initialisations.

\label{sec:extrapolation}
For extrapolation, we train on the central 60\% of each input interval and evaluate beyond its upper training boundary. This separate experiment uses all 50 targets and all 15 models. Banach has the lowest median error, \ExtrapBanachMedian, compared with \ExtrapSplineMedian{} for the Adam spline and \ExtrapTanhMedian{} for $\tanh$ (Table~\ref{tab:extrap}(b)). The rational control has the best average rank and the MLP the most per-equation wins. Banach's paired ratio against the spline is \ExtrapSplineRatio, with intervals including one, so the lower median does not establish a paired extrapolation advantage. The full protocol and paired intervals appear in the appendix.

We test whether the branches benefit from separate input scales and shifts by imposing $w_1=w_2$ and $b_1=b_2$ in Eq.~\eqref{eq:banach}, reducing eight parameters per edge to six. At the same power-gain initialisation distribution, tying worsens clean fitting (Table~\ref{tab:geometry}(c)). Other tied rows vary the initial gain and exponent; all learn $p$. Noisy results have no uniform ordering; Appendix~\ref{app:controls} reports both directions. This tests affine sharing with reduced capacity; signed gains do not enforce exact convex duality. Appendix~\ref{app:controls} explains each control.

\textbf{Beyond symbolic regression: tabular data.}\label{sec:beyond}
We evaluate on 12 PMLB regression datasets~\cite{romano2022pmlb}, using an 80/20 train/test split, standardized features and targets, and the same validation-selected Adam protocol and model widths. Table~\ref{tab:pmlb} retains the five-model comparison; Appendix~\ref{app:pmlb} gives all 15 models and dataset details. Tanh has the lowest mean rank in both rosters. Banach does not lead this tabular comparison, so the symbolic-regression findings do not imply a general advantage on tabular data.

\begin{table}[!t]
\centering\footnotesize\setlength{\tabcolsep}{3pt}
\caption{PMLB regression: median test NRMSE over five runs. Bold marks the best unrounded value among these five models; wins and ranks use this roster only. The full 15-model comparison is in Appendix~\ref{app:pmlb}. Here $n$ is the complete dataset size and $d$ the input dimension.}
\label{tab:pmlb}
\begin{tabular}{llrrrrrrr}
\toprule
Dataset & Domain & $d$ & $n$ & Spline & Tanh & Banach & $\ell^p$ & MLP \\
\midrule
Geyser & geophysics & 2 & 222 & 0.635 & 0.460 & 0.470 & \textbf{0.449} & 0.462 \\
Environment & environmental & 3 & 111 & 1.117 & \textbf{0.919} & 1.042 & 0.934 & 1.030 \\
ESL & ordinal & 4 & 488 & 0.409 & \textbf{0.373} & 0.538 & 0.407 & 0.425 \\
ERA & ordinal & 4 & 1000 & 0.815 & \textbf{0.808} & 0.808 & 0.817 & 0.815 \\
Cloud & meteorology & 5 & 108 & 0.574 & \textbf{0.399} & 0.890 & 0.664 & 0.401 \\
Machine CPU & computing & 6 & 209 & 0.525 & \textbf{0.293} & 0.826 & 0.555 & 0.579 \\
Sleuth case 2002 & biology & 6 & 147 & 1.548 & \textbf{0.992} & 1.086 & 1.180 & 1.169 \\
CPU & computing & 7 & 209 & 0.136 & 0.217 & \textbf{0.072} & 0.187 & 0.124 \\
PM10 & pollution & 7 & 500 & 1.224 & 0.914 & 0.947 & \textbf{0.870} & 1.160 \\
NO$_2$ & pollution & 7 & 500 & 0.976 & \textbf{0.685} & 0.831 & 0.726 & 0.890 \\
LA data & epidemiology & 10 & 508 & 1.030 & \textbf{0.627} & 0.831 & 0.643 & 0.803 \\
SWD & ordinal & 10 & 1000 & 0.836 & 0.794 & 0.816 & \textbf{0.791} & 0.830 \\
\midrule
Median NRMSE & \multicolumn{3}{c}{} & 0.825 & \textbf{0.656} & 0.829 & 0.695 & 0.809 \\
Wins & \multicolumn{3}{c}{} & 0 & \textbf{8} & 1 & 3 & 0 \\
Mean rank & \multicolumn{3}{c}{} & 4.08 & \textbf{1.58} & 3.42 & 2.58 & 3.33 \\
\bottomrule\end{tabular}

\end{table}

\section{Discussion and Conclusion}
\label{sec:discussion}
Learning a per-edge exponent provides a compact way to adapt the near-origin response and tail growth of a KAN edge. The experiments identify complementary regimes: the tested spline has the smallest aggregate errors with plentiful clean data, the combined Banach response remains competitive with few clean examples, and learned power-law responses remain competitive under stronger measurement noise. On the 18-equation core, fixed-exponent and branch-removal controls support learning the exponent and retaining both components under the tested settings. These findings support adapting the response family to the task; they do not establish a basis-independent advantage of Banach geometry.

Banach improves on slope-adaptive and additive-$\varepsilon$ controls on clean data, while rational edges attain lower aggregate error. Under heavy noise, the C$^1$ and additive-$\varepsilon$ power responses are not distinguished. For the exact power response, a positive exponent $g$ and the duality parameter $p$ describe the same family through $g=p-1$. Smoothing and optimisation coordinates also differ, so the comparison cannot isolate geometric interpretation. The learned exponent is a fitted shape descriptor, not an identifiable physical parameter.

The geometry supplies convex-potential interpretations in the stated cases, bounded-domain sensitivity calculations for the smooth response and an exact-map derivative-moment condition. These properties do not provide a general optimisation or generalisation theorem. Measurement-noise experiments do not establish adversarial robustness, and the empirical evaluation covers symbolic and small tabular regression. Parameter counts are reported without claiming that equal counts imply equal capacity or compute.

Our approach has two limitations, each suggesting a concrete future-work direction. First, the restricted per-edge form can limit finite-budget fitting; composing geometric responses with a small learnable basis component may expand local expressive power while retaining a geometric prior. Deeper compositions also require direct evaluation. Second, extending these results to larger tabular datasets and higher-dimensional modalities requires separate evidence.

\label{sec:mainend}
\clearpage
\section*{Acknowledgements and disclosures}
\textbf{Funding and acknowledgements.} This work was supported by the Brevan Howard Centre for Financial Analysis at Imperial Business School. We thank Dr.~Erik Chavez for providing access to computing resources.

\textbf{Competing interests.} The author declares no competing interests.

\textbf{Reproducibility.} Code is available at \url{https://github.com/seshkumar/geometric-kan-neurips2026}. The self-contained kit includes the fitting programs, retained experiment outputs and commands to regenerate the numerical tables and motivation figure.
\bibliographystyle{plain}
\bibliography{references}

@inproceedings{liu2024kan,
  author    = {Liu, Ziming and Wang, Yixuan and Vaidya, Sachin and Ruehle, Fabian and Halverson, James and Soljacic, Marin and Hou, Thomas Y. and Tegmark, Max},
  title     = {{KAN}: {K}olmogorov--{A}rnold Networks},
  booktitle = {The Thirteenth International Conference on Learning Representations (ICLR)},
  year      = {2025},
  url       = {https://openreview.net/forum?id=Ozo7qJ5vZi},
}

@misc{udrescu2020ai,
  author        = {Silviu-Marian Udrescu and Max Tegmark},
  title         = {{AI} {F}eynman: a Physics-Inspired Method for Symbolic Regression},
  year          = {2020},
  eprint        = {1905.11481},
  archivePrefix = {arXiv},
  primaryClass  = {physics.comp-ph},
  url           = {https://arxiv.org/abs/1905.11481},
  howpublished  = {arXiv preprint arXiv:1905.11481},
}

@misc{li2024rbfkan,
  author        = {Ziyao Li},
  title         = {{K}olmogorov-{A}rnold Networks are Radial Basis Function Networks},
  year          = {2024},
  eprint        = {2405.06721},
  archivePrefix = {arXiv},
  primaryClass  = {cs.LG},
  url           = {https://arxiv.org/abs/2405.06721},
  howpublished  = {arXiv preprint arXiv:2405.06721},
}

@misc{bozorgasl2024wav,
  author        = {Zavareh Bozorgasl and Hao Chen},
  title         = {Wav-{KAN}: Wavelet {K}olmogorov-{A}rnold Networks},
  year          = {2024},
  eprint        = {2405.12832},
  archivePrefix = {arXiv},
  primaryClass  = {cs.LG},
  url           = {https://arxiv.org/abs/2405.12832},
  howpublished  = {arXiv preprint arXiv:2405.12832},
}

@misc{ss2024chebyshev,
  author        = {Sidharth SS and Keerthana AR and Gokul R and Anas KP},
  title         = {Chebyshev Polynomial-Based {K}olmogorov-{A}rnold Networks: An Efficient Architecture for Nonlinear Function Approximation},
  year          = {2024},
  eprint        = {2405.07200},
  archivePrefix = {arXiv},
  primaryClass  = {cs.LG},
  url           = {https://arxiv.org/abs/2405.07200},
  howpublished  = {arXiv preprint arXiv:2405.07200},
}

@article{aghaei2024fkan,
  author  = {Alireza Afzal Aghaei},
  title   = {{fKAN}: Fractional {K}olmogorov--{A}rnold Networks with trainable {J}acobi basis functions},
  journal = {Neurocomputing},
  volume  = {623},
  pages   = {129414},
  year    = {2025},
  issn    = {0925-2312},
  doi     = {10.1016/j.neucom.2025.129414},
  url     = {https://www.sciencedirect.com/science/article/pii/S0925231225000864},
}

@misc{xu2024fourierkan,
  author        = {Jinfeng Xu and Zheyu Chen and Jinze Li and Shuo Yang and Wei Wang and Xiping Hu and Edith Ngai},
  title         = {Enhancing Graph Collaborative Filtering with {FourierKAN} Feature Transformation},
  year          = {2025},
  eprint        = {2406.01034},
  archivePrefix = {arXiv},
  primaryClass  = {cs.IR},
  url           = {https://arxiv.org/abs/2406.01034},
  howpublished  = {arXiv preprint arXiv:2406.01034},
}

@book{megginson1998banach,
  author    = {R. E. Megginson},
  title     = {An Introduction to {B}anach Space Theory},
  series    = {Graduate Texts in Mathematics},
  volume    = {183},
  publisher = {Springer},
  year      = {1998},
}

@book{cioranescu1990geometry,
  author    = {Cioranescu, Ioana},
  title     = {Geometry of {B}anach Spaces, Duality Mappings and Nonlinear Problems},
  series    = {Mathematics and Its Applications},
  volume    = {62},
  year      = {1990},
  publisher = {Kluwer Academic Publishers},
  address   = {Dordrecht},
  isbn      = {978-0-7923-0910-9},
  url       = {https://link.springer.com/book/10.1007/978-94-009-2121-4},
}

@book{bauschke2017convex,
  author    = {Heinz H. Bauschke and Patrick L. Combettes},
  title     = {Convex Analysis and Monotone Operator Theory in {H}ilbert Spaces},
  edition   = {2nd},
  publisher = {Springer},
  series    = {CMS Books in Mathematics},
  year      = {2017},
  isbn      = {978-3-319-48311-5},
  doi       = {10.1007/978-3-319-48311-5},
}

@misc{kingma2014adam,
  author        = {Diederik P. Kingma and Jimmy Ba},
  title         = {Adam: A Method for Stochastic Optimization},
  year          = {2014},
  eprint        = {1412.6980},
  archivePrefix = {arXiv},
  primaryClass  = {cs.LG},
  url           = {https://arxiv.org/abs/1412.6980},
  howpublished  = {arXiv preprint arXiv:1412.6980},
}

@article{cybenko1989,
  author  = {G. Cybenko},
  title   = {Approximation by superpositions of a sigmoidal function},
  journal = {Mathematics of Control, Signals and Systems},
  volume  = {2},
  number  = {4},
  pages   = {303--314},
  year    = {1989},
}

@article{hochreiter1997lstm,
  author  = {S. Hochreiter and J. Schmidhuber},
  title   = {Long short-term memory},
  journal = {Neural Computation},
  volume  = {9},
  number  = {8},
  pages   = {1735--1780},
  year    = {1997},
}

@inproceedings{cho2014learning,
  author    = {Cho, Kyunghyun and van Merri{\"e}nboer, Bart and Gulcehre, Caglar and Bahdanau, Dzmitry and Bougares, Fethi and Schwenk, Holger and Bengio, Yoshua},
  title     = {Learning Phrase Representations using {RNN} Encoder--Decoder for Statistical Machine Translation},
  booktitle = {Proceedings of the 2014 Conference on Empirical Methods in Natural Language Processing (EMNLP)},
  publisher = {Association for Computational Linguistics},
  address   = {Doha, Qatar},
  month     = oct,
  year      = {2014},
  pages     = {1724--1734},
  doi       = {10.3115/v1/D14-1179},
  url       = {https://aclanthology.org/D14-1179/},
}

@incollection{lecun1998efficient,
  author    = {LeCun, Yann and Bottou, L{\'e}on and Orr, Genevieve B. and M{\"u}ller, Klaus-Robert},
  title     = {Efficient {B}ack{P}rop},
  booktitle = {Neural Networks: Tricks of the Trade},
  series    = {Lecture Notes in Computer Science},
  publisher = {Springer-Verlag},
  address   = {Berlin, Heidelberg},
  pages     = {9--50},
  year      = {1998},
  isbn      = {3540653112},
}

@article{hornik1989,
  author  = {K. Hornik and M. Stinchcombe and H. White},
  title   = {Multilayer feedforward networks are universal approximators},
  journal = {Neural Networks},
  volume  = {2},
  number  = {5},
  pages   = {359--366},
  year    = {1989},
}

@article{romano2022pmlb,
  author        = {Romano, Joseph D and Le, Trang T and La Cava, William and Gregg, John T and Goldberg, Daniel J and Chakraborty, Praneel and Ray, Natasha L and Himmelstein, Daniel and Fu, Weixuan and Moore, Jason H},
  title         = {{PMLB} v1.0: an open-source dataset collection for benchmarking machine learning methods},
  journal       = {Bioinformatics},
  volume        = {38},
  number        = {3},
  pages         = {878--880},
  year          = {2022},
  publisher     = {Oxford University Press},
  doi           = {10.1093/bioinformatics/btab727},
}

@article{bach2017breaking,
  author        = {Francis Bach},
  title         = {Breaking the {C}urse of {D}imensionality with {C}onvex {N}eural {N}etworks},
  journal       = {Journal of Machine Learning Research},
  volume        = {18},
  number        = {19},
  pages         = {1--53},
  year          = {2017},
  url           = {http://jmlr.org/papers/v18/14-546.html},
}

@article{jagtap2020adaptive,
  title={Adaptive activation functions accelerate convergence in deep and physics-informed neural networks},
  author={Jagtap, Ameya D. and Kawaguchi, Kenji and Karniadakis, George Em},
  journal={Journal of Computational Physics},
  volume={404},
  pages={109136},
  year={2020},
  doi={10.1016/j.jcp.2019.109136}
}

@inproceedings{molina2020pade,
  title={Pad\'e Activation Units: End-to-end Learning of Flexible Activation Functions in Deep Networks},
  author={Molina, Alejandro and Schramowski, Patrick and Kersting, Kristian},
  booktitle={International Conference on Learning Representations (ICLR)},
  year={2020},
  eprint={1907.06732},
  archivePrefix={arXiv}
}

@inproceedings{boulle2020rational,
  title={Rational neural networks},
  author={Boull{\'e}, Nicolas and Nakatsukasa, Yuji and Townsend, Alex},
  booktitle = {Advances in Neural Information Processing Systems},
  volume = {33},
  pages = {14243--14253},
  year={2020},
  url = {https://proceedings.neurips.cc/paper/2020/file/a3f390d88e4c41f2747bfa2f1b5f87db-Paper.pdf}
}

@inproceedings{manessi2018learning,
  title={Learning Combinations of Activation Functions},
  author={Manessi, Franco and Rozza, Alessandro},
  booktitle={International Conference on Pattern Recognition (ICPR)},
  pages={61--66},
  doi={10.1109/ICPR.2018.8545362},
  year={2018},
  eprint={1801.09403},
  archivePrefix={arXiv}
}

@article{ismailov2026universality,
  author = {Ismailov, Vugar E.},
  title = {Necessary and sufficient conditions for universality of {Kolmogorov--Arnold} networks},
  journal = {arXiv preprint arXiv:2604.23765},
  year = {2026},
  doi = {10.48550/arXiv.2604.23765},
  url = {https://arxiv.org/abs/2604.23765}
}

@article{bartels2018unconditional,
  author  = {Bartels, S{\"o}ren and Diening, Lars and Nochetto, Ricardo H.},
  title   = {Unconditional Stability of Semi-Implicit Discretizations of Singular Flows},
  journal = {SIAM Journal on Numerical Analysis},
  volume  = {56},
  number  = {3},
  pages   = {1896--1914},
  year    = {2018},
  doi     = {10.1137/17M1159166}
}
\clearpage
\appendix
\raggedbottom
\section*{Appendix}
\makeatletter
\setlength{\@fptop}{0pt}
\setlength{\@fpsep}{12pt}
\setlength{\@fpbot}{0pt plus 1fil}
\makeatother
This appendix gives the experimental settings, additional results and benchmark definitions. The main clean, noise and scarcity comparisons use all 50 equations. The separate fixed-exponent, branch-removal, tied-affine and spline optimisation controls use only the 18-equation core. These populations are identified explicitly below and are not pooled.

\section{Spline--Adam optimisation controls}
\label{app:sweep}
The safe rational control is inspired by~\cite{molina2020pade}, with per-edge numerator/denominator degrees $3/2$, rather than the original degree-$5/4$ layer-shared design. Its linear denominator coefficient starts at $0.01$ and its quadratic coefficient at zero. The mixture uses unconstrained weights on $\tanh$, ReLU, sine and identity, rather than the affine/convex combinations in~\cite{manessi2018learning}; the slope control similarly adapts~\cite{jagtap2020adaptive} per edge. The Jacobi control fixes the Legendre case without fractional-order adaptation and does not reproduce fKAN~\cite{aghaei2024fkan}. The smoothing in Eq.~\eqref{eq:smooth} takes its functional form from~\cite{bartels2018unconditional}, whose analysis concerns $1\le p<2$; we use the same regularisation for all learned $p>1$.

The main 15-model comparison trains every model with Adam for 5,000 steps, using initial learning rate $0.01$, cosine decay and gradient-norm clipping at 10. The two spline rows use the same cubic B-spline implementation with $G=3$. The Lipschitz-regularised row adds $10^{-2}$ times the mean squared input-gradient norm to the fitting loss. This coefficient is fixed across equations and noise levels.

Additional optimisation controls use only the 18-equation core. They compare cubic pykan with $G=3$, 300 fixed L-BFGS steps and grid updates disabled against the specified Adam alternatives. Their clean medians are \CoreSpline{} for L-BFGS, \CoreSplineAdam{} for Adam and \CoreSplineLong{} for 10,000 Adam steps (Table~\ref{tab:ablation_full}(c)). The pykan and matched Adam splines are different implementations, so their scores are not interchangeable. These controls establish sensitivity to the chosen fitting procedure rather than a ranking invariant to optimisation.

\section{Architectures and parameter counts}
\label{app:parameters}
\label{app:params}
All KANs have widths $[d,K,1]$, where $d$ is the input dimension and $K$ is five for $d\le2$, four for $d=3,4$ and three otherwise. The MLP has one width-50 rectified-linear hidden layer for $d\le3$ and two for larger inputs. The table reports trainable parameter counts for each of the 15 models used in the 50-equation main comparison.

\begin{center}
\small\begin{tabular}{lrrrrr}
\toprule
Model & $d=2$ & $d=3$ & $d=4$ & $d=5$ & $d=6$ \\
\midrule
Banach & 120 & 128 & 160 & 144 & 168 \\
$\ell^p$ & 75 & 80 & 100 & 90 & 105 \\
Tanh & 60 & 64 & 80 & 72 & 84 \\
$\ell^p$ (additive $\varepsilon$) & 75 & 80 & 100 & 90 & 105 \\
Adaptive-slope Tanh & 75 & 80 & 100 & 90 & 105 \\
Activation mixture & 105 & 112 & 140 & 126 & 147 \\
Rational & 120 & 128 & 160 & 144 & 168 \\
Spline (Adam) & 165 & 176 & 220 & 198 & 231 \\
Spline (Adam, Lipschitz-reg.) & 165 & 176 & 220 & 198 & 231 \\
Chebyshev & 105 & 112 & 140 & 126 & 147 \\
RBF & 105 & 112 & 140 & 126 & 147 \\
Fourier & 120 & 128 & 160 & 144 & 168 \\
Wavelet & 90 & 96 & 120 & 108 & 126 \\
Jacobi & 105 & 112 & 140 & 126 & 147 \\
MLP & 201 & 251 & 2851 & 2901 & 2951 \\
\bottomrule
\end{tabular}

\end{center}
Banach and the rational response have the same parameter count at each dimension; the learned activation mixture uses fewer. This allows comparisons with similarly sized parametric responses, but is not an equal-parameter comparison across the complete roster. The matched Adam spline and pykan allocate parameters differently. Fixed grids and other buffers are not counted as trainable weights. Parameter counts do not equate training cost, hyperparameter-search effort or effective capacity.

\section{Spline regularisation under noise}
\label{app:spline_reg}
This additional spline cross-validation study uses only the 18-equation core for clean fitting, noisy fitting and noisy scarcity. It is separate from the 15-model, 50-equation main comparison in Table~\ref{tab:matched}.

The core spline baseline is cubic pykan with $G=3$ and the fixed L-BFGS budget described in Appendix~\ref{app:sweep}. Cross-validation (CV) selects its regularisation coefficient from $\{0,10^{-3},10^{-2},10^{-1}\}$ using five training-only folds, followed by fitting on the full training pool. The Adam models retain a validation carve, so the fitting-data allocation and compute are not matched. No test-selected regularisation result is used.

\begin{table}[!htbp]
\centering\footnotesize\setlength{\tabcolsep}{1.5pt}
\caption{Spline cross-validation study on the 18-equation core only. Clean wins and ranks compare all seven displayed models. Remaining entries are median NRMSE for the separate noise and noisy-scarcity runs. Bold marks lower error and rank or higher wins.}
\label{tab:baseline}\label{tab:noise}\label{tab:scarcity}
\begin{tabular}{lrrrrrrrrrrrr}
\toprule
& \multicolumn{3}{c}{Clean} & \multicolumn{4}{c}{Noise $\sigma$ ($n=500$)} & \multicolumn{5}{c}{Training size $n$ ($\sigma=0.3$)} \\
\cmidrule(lr){2-4}\cmidrule(lr){5-8}\cmidrule(lr){9-13}
Model & NRMSE & Wins & Rank & 0.1 & 0.3 & 0.5 & 1.0 & 50 & 100 & 200 & 500 & 1000 \\
\midrule
Spline G=3 & 0.037 & 2.5 & 3.03 & 0.075 & 0.201 & 0.347 & 0.784 & 1.076 & 0.852 & 0.402 & 0.201 & 0.114 \\
CV-regularized G3 & 0.040 & \textbf{5.5} & \textbf{2.47} & 0.057 & 0.131 & 0.187 & 0.329 & 0.427 & 0.254 & 0.184 & 0.131 & 0.084 \\
Tanh & 0.070 & 0 & 6.06 & 0.079 & 0.137 & 0.195 & 0.324 & 0.410 & 0.292 & 0.203 & 0.137 & 0.104 \\
Banach & \textbf{0.031} & 4 & 2.67 & \textbf{0.049} & \textbf{0.113} & 0.189 & 0.323 & 0.393 & \textbf{0.242} & 0.172 & \textbf{0.113} & \textbf{0.074} \\
$\ell^p$ & 0.088 & 0 & 5.44 & 0.094 & 0.127 & \textbf{0.160} & \textbf{0.287} & \textbf{0.386} & 0.261 & \textbf{0.166} & 0.127 & 0.094 \\
MLP & 0.060 & 4 & 4.78 & 0.081 & 0.163 & 0.245 & 0.383 & 0.397 & 0.319 & 0.223 & 0.163 & 0.122 \\
Chebyshev & 0.042 & 2 & 3.56 & 0.060 & 0.141 & 0.231 & 0.524 & 0.657 & 0.430 & 0.267 & 0.141 & 0.083 \\
\bottomrule
\end{tabular}

\end{table}
Regularisation materially changes the comparison: at $\sigma=1$, CV reduces the core spline median from \CoreHeavySpline{} to \CoreHeavyCV, versus \CoreHeavyLp{} for $\ell^p$. At training-pool size 50 and $\sigma=0.3$, the medians for $\ell^p$, CV spline and Banach are \CoreScarceLp, \CoreScarceCV{} and \CoreScarceBanach. These close medians do not establish statistical equivalence. A lower aggregate median is therefore insufficient to claim paired superiority; equation-paired intervals are reported separately below.

\section{Fixed-exponent and power-branch controls}
\label{app:controls}
The fixed-exponent and branch-removal results in Table~\ref{tab:ablation} use only the 18-equation core, five runs per condition and training-pool size 500. Clean and noisy-label conditions are reported separately. The power-only fixed values are $p\in\{1.5,2,2.5,3,4\}$; the combined response is compared with fixed $p=2$. All other settings follow the corresponding core protocol.

\begin{table}[!htbp]
\centering\small
\begin{minipage}{.32\linewidth}\centering\resizebox{\linewidth}{!}{\begin{tabular}{lr}
\toprule
Model & Median NRMSE \\
\midrule
$\ell^p$, fixed $p=1.5$ & 0.172 \\
$\ell^p$, fixed $p=2$ & 0.482 \\
$\ell^p$, fixed $p=2.5$ & 0.178 \\
$\ell^p$, fixed $p=3$ & 0.181 \\
$\ell^p$, fixed $p=4$ & 0.198 \\
$\ell^p$, learned $p$ & 0.088 \\
Banach, fixed $p=2$ & 0.063 \\
Banach, learned $p$ & \textbf{0.031} \\
\bottomrule
\end{tabular}
}\\[2pt]{\footnotesize (a) clean}
\end{minipage}\hfill
\begin{minipage}{.32\linewidth}\centering\resizebox{\linewidth}{!}{\begin{tabular}{lr}
\toprule
Model & Median NRMSE \\
\midrule
$\ell^p$, fixed $p=1.5$ & 0.216 \\
$\ell^p$, fixed $p=2$ & 0.483 \\
$\ell^p$, fixed $p=2.5$ & 0.203 \\
$\ell^p$, fixed $p=3$ & 0.204 \\
$\ell^p$, fixed $p=4$ & 0.221 \\
$\ell^p$, learned $p$ & 0.127 \\
Banach, fixed $p=2$ & 0.140 \\
Banach, learned $p$ & \textbf{0.113} \\
\bottomrule
\end{tabular}
}\\[2pt]{\footnotesize (b) $\sigma=0.3$}
\end{minipage}\hfill
\begin{minipage}{.32\linewidth}\centering\resizebox{\linewidth}{!}{\begin{tabular}{lr}
\toprule
Model & Clean NRMSE \\
\midrule
Banach & \textbf{0.031} \\
Banach, power branch off & 0.082 \\
Tanh & 0.070 \\
$\ell^p$ & 0.088 \\
Spline G=3 & 0.037 \\
Spline Adam & 0.042 \\
Spline Adam (10k) & 0.042 \\
\bottomrule
\end{tabular}
}\\[2pt]{\footnotesize (c) clean}
\end{minipage}
\par\medskip
{\footnotesize\setlength{\tabcolsep}{3pt}\begin{tabular}{lrrrrrrrrr}
\toprule
Affine variant & $\sigma=0$ & $\sigma=0.1$ & $\sigma=0.3$ & $\sigma=0.5$ & $\sigma=1$ & $n=50$ & $n=100$ & $n=200$ & $n=1000$ \\
\midrule
Independent & \textbf{0.031} & 0.049 & \textbf{0.113} & 0.189 & 0.323 & 0.393 & \textbf{0.242} & 0.172 & 0.074 \\
Tied, matched gain & 0.037 & \textbf{0.048} & 0.117 & 0.180 & 0.313 & 0.445 & 0.272 & 0.184 & \textbf{0.074} \\
Tied & 0.037 & 0.052 & 0.117 & \textbf{0.171} & 0.320 & 0.454 & 0.290 & \textbf{0.168} & 0.078 \\
Tied, $p_0=2$ & 0.059 & 0.061 & 0.120 & 0.171 & \textbf{0.308} & \textbf{0.352} & 0.280 & 0.179 & 0.082 \\
\bottomrule\end{tabular}
}\par\smallskip (d) tied-affine noise and noisy scarcity
\caption{Complete control tables on the 18-equation core. (a) Clean and (b) noisy fixed-exponent sweeps; (c) branch-removal and optimiser controls; (d) the separate four-arm tied-affine study, with five noise levels at $n=500$, followed by four sample sizes at $\sigma=0.3$. All tested rows are retained. The main-text fixed-$p$ rows show the lowest benchmark median separately per condition; this is a descriptive summary, not validation-selected hyperparameter tuning.}
\label{tab:ablation_full}
\end{table}

Learning the exponent changes one scalar per edge relative to its fixed counterpart. The comparison tests whether adapting the response shape helps within this family; it is not a general capacity-matched comparison with every alternative activation. Disabling the power branch leaves the saturating branch trainable but also removes the trainable influence of the power-branch parameters. The intact family therefore benefits from the branch in these fits, while the experiment does not hold effective capacity fixed.

\subsection{Tied-affine controls}
The question is whether the bounded and power-law branches benefit from choosing their input scales and shifts separately. The independent response in Eq.~\eqref{eq:banach} has two affine arguments. The tied response instead uses one shared argument $u=wx+b$:
\[
 \phi_{\mathrm{tied}}(x)=a_1\tanh(u)+a_2J_{p,\mathrm{sm}}^\varepsilon(u)+d.
\]
Its six trainable scalars are $a_1,a_2,w,b,d$ and the exponent coordinate, rather than the independent response's eight. Thus this comparison tests the restriction imposed by sharing, including the loss of two parameters; it is not an equal-budget comparison.

The four rows of Tables~\ref{tab:geometry}(c) and~\ref{tab:ablation_full}(d) have the following roles. Here $p_0$ denotes the initial exponent, which remains trainable in every arm.
\begin{itemize}
\item \textbf{Independent:} separate branch affines, $p_0=2.5$, and power gains $a_2$ initialised uniformly on $[-0.5,0.5]$; this is the reference.
\item \textbf{Tied, matched gain:} shared affines with the same power-gain initialisation distribution and $p_0$ as the reference. This is the primary affine-sharing control.
\item \textbf{Tied:} shared affines and $p_0=2.5$, with a smaller initial power gain, uniform on $[-0.1,0.1]$. Relative to the preceding row, this tests sensitivity to the initial power-branch amplitude.
\item \textbf{Tied, $p_0=2$:} the same smaller-gain tied construction, initialised at exponent two. Relative to the preceding row, this tests exponent initialisation, not a fixed-exponent model.
\end{itemize}
The independent reference is the same retained Banach fit used in Tables~\ref{tab:ablation}(b) and~\ref{tab:ablation_full}(b), rather than a fresh rerun. All arms use the C$^1$ map, the same dimension-dependent KAN widths, 5,000 Adam steps, paired data and seeds, and training-only validation selection. Matching an initialisation distribution does not give identical initial functions. The gains remain signed and unconstrained: shared affines alone therefore do not impose the nonnegative-gain convex-potential construction discussed in Section~\ref{sec:banachkan}.

The primary comparison gives a clean equation-paired tied/independent error ratio of \TiedCleanRatio, with a family-wise interval of \TiedCleanInterval; values above one favour independent affines. This supports retaining separate branch scales and shifts for clean fitting under the tested budget. The smaller-gain and exponent-two rows show that noise results also depend on initialisation, so differences between those rows and the reference cannot be attributed to affine sharing alone. Noise and scarcity comparisons have no uniform ordering. At $\sigma=0.5$, the matched-gain tied/independent ratio is \TiedNoiseRatio{} with a family-wise interval of \TiedNoiseInterval; the lower point estimate does not establish an adjusted advantage. Each condition has its own three-comparison bootstrap family, with no correction across conditions. Table~\ref{tab:ablation_full}(d) retains the full noise and noisy-scarcity grid.

\section{Per-equation clean fitting on all 50 targets}
The following table uses exactly the 15-model roster of the main comparison and reports the median test NRMSE over five runs within each equation. Model abbreviations are: B, Banach; L, $\ell^p$; T, $\tanh$; P, $\ell^p$ (additive $\varepsilon$); A, adaptive-slope $\tanh$; M, activation mixture; Q, rational/Pad\'e; S, Adam spline; $S_L$, Lipschitz-regularised Adam spline; C, Chebyshev; R, radial basis; F, Fourier; W, wavelet; J, Jacobi; N, MLP.
\begin{table}[!htbp]\centering\scriptsize\setlength{\tabcolsep}{2pt}
\caption{Clean test NRMSE by equation: median across five runs, for the same 15 models as Table~\ref{tab:matched}. Bold marks the lowest unrounded value within each row.}\label{tab:per_equation}
\begin{tabular}{lrrrrrrrrrrrrrrr}
\toprule Target & B & L & T & P & A & M & Q & S & $S_L$ & C & R & F & W & J & N \\ \midrule
I.12.1 & 0.005 & 0.006 & 0.009 & 0.007 & 0.009 & 0.008 & 0.006 & \textbf{0.002} & 0.016 & 0.004 & 0.004 & 0.004 & 0.006 & 0.005 & 0.036 \\
I.14.4 & 0.005 & 0.007 & 0.011 & 0.006 & 0.010 & 0.012 & 0.005 & \textbf{0.002} & 0.018 & 0.006 & 0.005 & 0.005 & 0.006 & 0.005 & 0.076 \\
I.25.13 & 0.008 & 0.017 & 0.014 & 0.023 & 0.017 & 0.014 & 0.008 & \textbf{0.005} & 0.049 & 0.009 & 0.009 & 0.010 & 0.013 & 0.009 & 0.019 \\
I.34.14 & 0.017 & 0.035 & 0.046 & 0.037 & 0.044 & 0.026 & 0.015 & \textbf{0.012} & 0.036 & 0.025 & 0.016 & 0.027 & 0.025 & 0.020 & 0.035 \\
I.14.3 & 0.018 & 0.031 & 0.041 & 0.047 & 0.046 & 0.025 & 0.020 & \textbf{0.013} & 0.032 & 0.019 & 0.014 & 0.023 & 0.021 & 0.018 & 0.057 \\
I.18.12 & 0.033 & 0.161 & 0.066 & 0.108 & 0.069 & 0.056 & 0.034 & \textbf{0.015} & 0.034 & 0.022 & 0.029 & 0.017 & 0.032 & 0.026 & 0.084 \\
I.34.10 & 0.029 & 0.094 & 0.031 & 0.079 & 0.033 & 0.033 & \textbf{0.018} & 0.019 & 0.069 & 0.029 & 0.037 & 0.034 & 0.035 & 0.030 & 0.028 \\
I.47.23 & 0.011 & 0.022 & 0.024 & 0.028 & 0.026 & 0.017 & \textbf{0.010} & 0.010 & 0.037 & 0.014 & 0.013 & 0.018 & 0.017 & 0.014 & 0.033 \\
III.17.37 & 0.058 & 0.105 & 0.278 & 0.106 & 0.270 & 0.043 & 0.043 & 0.030 & 0.076 & 0.056 & 0.050 & \textbf{0.025} & 0.067 & 0.063 & 0.113 \\
I.8.14 & 0.138 & 0.221 & 0.203 & 0.256 & 0.227 & 0.122 & 0.165 & 0.097 & 0.154 & 0.201 & 0.189 & 0.136 & 0.204 & 0.199 & \textbf{0.023} \\
I.13.4 & \textbf{0.011} & 0.027 & 0.042 & 0.025 & 0.041 & 0.022 & 0.012 & 0.011 & 0.023 & 0.016 & 0.020 & 0.018 & 0.026 & 0.016 & 0.036 \\
I.29.16 & 0.177 & 0.299 & 0.181 & 0.275 & 0.360 & 0.168 & 0.159 & 0.165 & 0.162 & 0.268 & 0.180 & 0.259 & 0.278 & 0.179 & \textbf{0.063} \\
III.4.32 & \textbf{0.021} & 0.046 & 0.074 & 0.060 & 0.069 & 0.074 & 0.046 & 0.037 & 0.109 & 0.056 & 0.066 & 0.053 & 0.074 & 0.069 & 0.073 \\
I.12.11 & 0.063 & 0.130 & 0.241 & 0.125 & 0.240 & 0.070 & 0.075 & 0.080 & \textbf{0.063} & 0.081 & 0.071 & 0.077 & 0.091 & 0.086 & 0.096 \\
I.44.4 & 0.056 & 0.082 & 0.300 & 0.099 & 0.408 & 0.107 & 0.077 & \textbf{0.049} & 0.100 & 0.106 & 0.073 & 0.110 & 0.124 & 0.117 & 0.096 \\
II.35.21 & 0.102 & 0.136 & 0.145 & 0.151 & 0.168 & 0.069 & 0.064 & \textbf{0.039} & 0.048 & 0.062 & 0.060 & 0.069 & 0.088 & 0.058 & 0.046 \\
III.14.14 & \textbf{0.086} & 0.178 & 0.142 & 0.165 & 0.238 & 0.131 & 0.110 & 0.110 & 0.317 & 0.120 & 0.126 & 0.125 & 0.146 & 0.152 & 0.112 \\
I.40.1 & 0.070 & 0.218 & 0.115 & 0.210 & 0.177 & 0.060 & 0.042 & \textbf{0.039} & 0.070 & 0.057 & 0.058 & 0.060 & 0.079 & 0.051 & 0.094 \\
II.6.15a & 0.112 & 0.201 & 0.262 & 0.311 & 0.348 & 0.098 & 0.094 & \textbf{0.073} & 0.217 & 0.091 & 0.137 & 0.076 & 0.140 & 0.092 & 0.210 \\
I.12.4 & 0.011 & 0.023 & 0.021 & 0.023 & 0.021 & 0.015 & 0.015 & \textbf{0.006} & 0.085 & 0.010 & 0.012 & 0.018 & 0.019 & 0.013 & 0.024 \\
I.10.7 & 0.008 & 0.014 & 0.014 & 0.015 & 0.013 & 0.015 & 0.010 & \textbf{0.004} & 0.155 & 0.008 & 0.007 & 0.020 & 0.011 & 0.008 & 0.040 \\
I.6.20a & 0.015 & 0.067 & 0.016 & 0.076 & 0.015 & 0.017 & 0.018 & \textbf{0.004} & 0.064 & 0.009 & 0.010 & 0.008 & 0.008 & 0.008 & 0.211 \\
F.43 & 0.016 & 0.033 & 0.046 & 0.049 & 0.046 & 0.025 & 0.019 & \textbf{0.013} & 0.030 & 0.019 & 0.014 & 0.020 & 0.023 & 0.019 & 0.054 \\
I.16.6 & 0.021 & 0.035 & 0.030 & 0.039 & 0.031 & 0.023 & \textbf{0.009} & 0.009 & 0.050 & 0.014 & 0.017 & 0.014 & 0.017 & 0.015 & 0.028 \\
I.27.6 & 0.017 & 0.071 & 0.032 & 0.063 & 0.031 & 0.031 & 0.025 & \textbf{0.011} & 0.041 & 0.023 & 0.019 & 0.022 & 0.023 & 0.023 & 0.030 \\
II.2.42 & 0.015 & 0.024 & 0.037 & 0.034 & 0.042 & 0.024 & 0.017 & \textbf{0.014} & 0.063 & 0.025 & 0.018 & 0.029 & 0.029 & 0.026 & 0.029 \\
II.15.4 & 0.060 & 0.106 & 0.162 & 0.111 & 0.141 & 0.041 & 0.029 & \textbf{0.019} & 0.038 & 0.037 & 0.034 & 0.031 & 0.037 & 0.045 & 0.088 \\
III.15.12 & 0.626 & 0.843 & 0.645 & 0.819 & 0.610 & 0.603 & 0.595 & 0.485 & \textbf{0.415} & 0.546 & 0.584 & 0.620 & 0.548 & 0.561 & 0.846 \\
I.18.16 & 0.071 & 0.105 & 0.100 & 0.150 & 0.207 & 0.073 & 0.062 & \textbf{0.024} & 0.042 & 0.036 & 0.028 & 0.077 & 0.070 & 0.037 & 0.091 \\
II.11.17 & 0.070 & 0.237 & 0.222 & 0.241 & 0.379 & 0.107 & \textbf{0.064} & 0.068 & 0.105 & 0.110 & 0.081 & 0.112 & 0.149 & 0.080 & 0.133 \\
I.37.4 & 0.031 & 0.051 & 0.070 & 0.048 & 0.071 & 0.026 & 0.017 & \textbf{0.008} & 0.028 & 0.023 & 0.022 & 0.019 & 0.021 & 0.021 & 0.062 \\
I.26.2 & 0.009 & 0.010 & 0.011 & 0.008 & 0.008 & 0.013 & 0.008 & \textbf{0.002} & 0.168 & 0.003 & 0.004 & 0.009 & 0.006 & 0.004 & 0.012 \\
F.33 & 0.351 & 0.668 & 0.597 & 0.647 & 0.634 & 0.363 & 0.348 & 0.358 & 0.308 & 0.339 & 0.338 & \textbf{0.275} & 0.397 & 0.381 & 0.315 \\
F.34 & 0.350 & 0.380 & 0.598 & 0.446 & 0.600 & \textbf{0.221} & 0.295 & 0.319 & 0.287 & 0.401 & 0.319 & 0.391 & 0.504 & 0.345 & 0.245 \\
III.10.19 & 0.017 & 0.027 & 0.041 & 0.021 & 0.043 & 0.028 & \textbf{0.014} & 0.022 & 0.029 & 0.021 & 0.025 & 0.020 & 0.025 & 0.017 & 0.029 \\
I.15.3t & 0.018 & 0.039 & 0.028 & 0.036 & 0.038 & 0.024 & 0.019 & \textbf{0.015} & 0.030 & 0.024 & 0.022 & 0.023 & 0.037 & 0.024 & 0.017 \\
I.24.6 & 0.044 & 0.067 & 0.077 & 0.076 & 0.084 & 0.049 & 0.036 & \textbf{0.025} & 0.041 & 0.045 & 0.028 & 0.052 & 0.060 & 0.042 & 0.059 \\
I.18.4 & 0.024 & 0.053 & 0.048 & 0.063 & 0.046 & 0.023 & 0.015 & \textbf{0.013} & 0.028 & 0.031 & 0.021 & 0.021 & 0.031 & 0.037 & 0.030 \\
I.13.12 & 0.110 & 0.084 & 0.279 & 0.100 & 0.333 & 0.114 & 0.089 & \textbf{0.064} & 0.134 & 0.119 & 0.092 & 0.104 & 0.136 & 0.131 & 0.094 \\
F.40 & 0.038 & 0.074 & 0.084 & 0.070 & 0.081 & 0.061 & \textbf{0.030} & 0.037 & 0.077 & 0.055 & 0.037 & 0.069 & 0.055 & 0.055 & 0.043 \\
I.39.11 & 0.018 & 0.039 & 0.031 & 0.048 & 0.035 & 0.041 & 0.026 & \textbf{0.015} & 0.260 & 0.026 & 0.021 & 0.034 & 0.022 & 0.022 & 0.057 \\
I.48.20 & 0.009 & 0.011 & 0.018 & 0.013 & 0.018 & 0.015 & 0.008 & \textbf{0.006} & 0.016 & 0.009 & 0.010 & 0.010 & 0.013 & 0.012 & 0.050 \\
II.11.27 & 0.051 & 0.156 & 0.077 & 0.156 & 0.075 & 0.062 & 0.049 & \textbf{0.021} & 0.076 & 0.046 & 0.042 & 0.029 & 0.082 & 0.048 & 0.050 \\
F.44 & 0.023 & 0.093 & 0.041 & 0.118 & 0.042 & 0.027 & 0.017 & \textbf{0.005} & 0.033 & 0.018 & 0.013 & 0.025 & 0.025 & 0.019 & 0.036 \\
F.45 & 0.020 & 0.051 & 0.043 & 0.064 & 0.051 & 0.028 & 0.019 & \textbf{0.010} & 0.030 & 0.022 & 0.020 & 0.015 & 0.021 & 0.020 & 0.020 \\
F.46 & 0.086 & 0.136 & 0.137 & 0.130 & 0.157 & 0.121 & 0.106 & \textbf{0.050} & 0.105 & 0.104 & 0.101 & 0.106 & 0.119 & 0.092 & 0.067 \\
F.31 & 0.075 & 0.068 & 0.132 & 0.078 & 0.167 & 0.083 & \textbf{0.050} & 0.058 & 0.127 & 0.112 & 0.115 & 0.078 & 0.152 & 0.113 & 0.077 \\
F.48 & 0.039 & 0.060 & 0.063 & 0.067 & 0.092 & 0.046 & 0.030 & \textbf{0.019} & 0.040 & 0.034 & 0.023 & 0.032 & 0.035 & 0.035 & 0.048 \\
F.49 & 0.152 & 0.185 & 0.241 & 0.193 & 0.272 & 0.090 & 0.147 & \textbf{0.055} & 0.065 & 0.071 & 0.101 & 0.074 & 0.121 & 0.082 & 0.087 \\
III.21.1 & \textbf{0.036} & 0.072 & 0.145 & 0.098 & 0.221 & 0.093 & 0.109 & 0.075 & 0.131 & 0.101 & 0.078 & 0.113 & 0.120 & 0.113 & 0.109 \\
\midrule
Median NRMSE & 0.032 & 0.069 & 0.068 & 0.076 & 0.070 & 0.042 & 0.029 & \textbf{0.019} & 0.063 & 0.033 & 0.028 & 0.030 & 0.036 & 0.036 & 0.057 \\
Wins & 4 & 0 & 0 & 0 & 0 & 1 & 7 & \textbf{32} & 2 & 0 & 0 & 2 & 0 & 0 & 2 \\
Mean rank & 5.16 & 11.10 & 12.30 & 11.98 & 12.94 & 8.36 & 4.48 & \textbf{1.82} & 9.98 & 5.90 & 4.82 & 6.48 & 8.96 & 6.22 & 9.50 \\
\bottomrule\end{tabular}
\end{table}

\section{Noise and small-sample results on all 50 targets}
\label{app:matched_details}
The complete mean-rank table below accompanies the NRMSE values in Table~\ref{tab:matched}. Every column compares all 15 models on all 50 equations. Noise is added to standardized training labels at fixed training-pool size 500; test targets remain clean. The sample-size sweep uses clean fitting labels. The clean $n=500$ fits provide the reference for both noise and sample-size comparisons, rather than additional independent experiments.

\begin{table}[!htbp]
\centering\small
\caption{Mean rank across all 50 targets among the same 15 models as Table~\ref{tab:matched}. Scores are first aggregated over runs within each equation. Lower is better; bold marks the best rank in each column. The spline rows use Adam.}
\label{tab:matched_ranks}\begin{tabular}{lrrrrrrrrr}
\toprule
& Clean & \multicolumn{4}{c}{Noise $\sigma$} & \multicolumn{4}{c}{Training size $n$} \\
\cmidrule(lr){3-6}\cmidrule(lr){7-10}
Model & 500 & 0.1 & 0.3 & 0.5 & 1.0 & 50 & 100 & 200 & 1000 \\
\midrule
Banach & 5.2 & 4.1 & 4.7 & 4.8 & 5.0 & 4.1 & 4.1 & 5.1 & 6.2 \\
$\ell^p$ & 11.1 & 9.0 & 6.0 & \textbf{3.3} & \textbf{2.4} & 6.0 & 7.2 & 9.6 & 12.0 \\
\midrule
Tanh & 12.3 & 9.6 & 6.1 & 5.3 & 3.8 & 7.2 & 9.2 & 11.3 & 12.4 \\
$\ell^p$ (additive $\varepsilon$) & 12.0 & 10.0 & 5.8 & 3.8 & \textbf{2.4} & 6.6 & 7.9 & 10.5 & 12.7 \\
Adaptive-slope Tanh & 12.9 & 11.5 & 8.7 & 7.8 & 6.2 & 9.4 & 11.3 & 12.0 & 12.9 \\
Activation mixture & 8.4 & 7.4 & 5.3 & 5.2 & 5.3 & 5.3 & 6.9 & 8.0 & 8.9 \\
Rational & 4.5 & \textbf{3.3} & \textbf{3.3} & 3.8 & 4.8 & \textbf{3.0} & \textbf{2.9} & 4.1 & 5.0 \\
Spline (Adam) & \textbf{1.8} & 6.5 & 11.6 & 13.6 & 14.8 & 7.4 & 5.2 & \textbf{3.2} & \textbf{1.3} \\
Spline (Adam, Lipschitz-reg.) & 10.0 & 9.6 & 10.0 & 10.6 & 10.0 & 11.2 & 10.9 & 11.0 & 9.1 \\
Chebyshev & 5.9 & 6.8 & 8.7 & 9.5 & 10.9 & 8.6 & 6.8 & 5.4 & 5.8 \\
RBF & 4.8 & 5.8 & 7.1 & 8.1 & 8.8 & 7.7 & 6.3 & 4.9 & 4.8 \\
Fourier & 6.5 & 9.0 & 12.1 & 12.7 & 13.3 & 14.2 & 12.4 & 7.9 & 6.4 \\
Wavelet & 9.0 & 10.0 & 11.9 & 12.5 & 12.2 & 12.8 & 11.7 & 10.7 & 8.2 \\
Jacobi & 6.2 & 6.8 & 8.6 & 10.1 & 11.3 & 8.6 & 7.9 & 6.4 & 5.8 \\
MLP & 9.5 & 10.5 & 9.9 & 8.8 & 8.9 & 7.8 & 9.3 & 9.8 & 8.5 \\
\bottomrule
\end{tabular}

\end{table}

The matched experiments include adaptive activations and fixed-basis comparators under one fitting protocol. They separate the clean-data strength of the spline from the smaller-sample strength of Banach and the stronger-noise performance of the learned power responses. Mean rank and median NRMSE summarise different aspects of the equation-level errors; neither replaces a paired comparison.

\subsection{Paired comparisons}
\label{app:paired}
For each equation, we first take the median NRMSE across runs. A paired ratio is the geometric mean across equations of the first model's error divided by the second's. Ratios below one favour the first model. Bootstrap samples resample equations with replacement while preserving model pairing; individual runs are not the resampling units. The percentile bootstrap uses 100,000 draws and seed 20260929. Ordinary 95\% intervals and Bonferroni-adjusted intervals are reported separately.

The full 50-equation matched family contains 90 declared comparisons: two proposed power responses against five comparator models in nine clean, noisy or sample-size conditions. The separate unified extrapolation family contains 26 comparisons: each of the two proposed power responses against the other 13 models. Every comparison is evaluated, with no reserved slots; the table displays a subset. Intervals describe variation within this finite equation collection, not all scientific tasks; they are approximate bootstrap intervals rather than exact finite-sample guarantees. These comparison families are specified for the analysis and are not claimed to be preregistered.

\begin{table}[!htbp]
\centering\small
\caption{Equation-paired geometric-mean NRMSE ratios on all 50 targets. Below one favours the first model. Family-wise intervals adjust within the matched or extrapolation comparison family as specified in the text.}
\label{tab:matched_paired}\begin{tabular}{llccc}
\toprule
Setting & Comparison & Ratio & 95\% interval & Family-wise \\
\midrule
clean & Banach / Tanh & 0.50 & [0.45, 0.55] & [0.41, 0.59] \\
clean & Banach / Adaptive-slope Tanh & 0.44 & [0.40, 0.50] & [0.36, 0.54] \\
clean & Banach / $\ell^p$ (additive $\varepsilon$) & 0.49 & [0.44, 0.54] & [0.40, 0.59] \\
clean & Banach / Spline (Adam) & 1.56 & [1.37, 1.79] & [1.24, 1.98] \\
clean & Banach / Spline (Adam, Lipschitz-reg.) & 0.53 & [0.42, 0.67] & [0.34, 0.77] \\
$\sigma=1$ & $\ell^p$ / $\ell^p$ (additive $\varepsilon$) & 1.03 & [1.00, 1.06] & [0.98, 1.08] \\
$\sigma=1$ & $\ell^p$ / Tanh & 0.87 & [0.83, 0.92] & [0.79, 0.95] \\
$\sigma=1$ & $\ell^p$ / Spline (Adam, Lipschitz-reg.) & 0.55 & [0.51, 0.59] & [0.48, 0.63] \\
$n=50$ & Banach / Spline (Adam) & 0.73 & [0.66, 0.81] & [0.61, 0.88] \\
$n=50$ & Banach / Tanh & 0.72 & [0.65, 0.80] & [0.60, 0.85] \\
$n=50$ & Banach / Adaptive-slope Tanh & 0.63 & [0.55, 0.71] & [0.50, 0.77] \\
$n=1000$ & Banach / Spline (Adam) & 2.01 & [1.81, 2.24] & [1.68, 2.44] \\
Extrapolation & Banach / Tanh & 0.57 & [0.40, 0.83] & [0.32, 1.03] \\
Extrapolation & Banach / Spline (Adam) & 0.86 & [0.59, 1.22] & [0.48, 1.48] \\
Extrapolation & Banach / $\ell^p$ (additive $\varepsilon$) & 0.83 & [0.70, 0.99] & [0.64, 1.10] \\
\bottomrule\end{tabular}

\end{table}

The auxiliary spline and within-family controls use a separate family of 24 comparisons: one 50-equation clean comparison, ten core noise comparisons, ten core noisy-scarcity comparisons, two core learned/fixed-exponent comparisons and one core branch-removal comparison. Its results remain separate from the 90-comparison full-benchmark family.
{\small\begin{table}[!htbp]\centering\small
\caption{Auxiliary equation-paired NRMSE ratios. Values below one favour the first model; values above one favour the second. Intervals are 95\% bootstrap intervals; the final column adjusts across all 24 comparisons. CV spline denotes CV-regularized G3.}\label{tab:paired}
\begin{tabular}{llrrr}
\toprule Condition & Comparison (first / second) & Ratio & 95\% interval & Adjusted interval \\ \midrule
\multicolumn{5}{l}{\textbf{Clean fitting: all 50 equations}}\\*
Clean & Banach / Spline G=3 & 1.299 & [1.064, 1.602] & [0.956, 1.819] \\
\multicolumn{5}{l}{\textbf{Label noise: 18-equation core, $n=500$}}\\*
$\sigma=0$ & Banach / CV spline & 1.241 & [0.906, 1.703] & [0.758, 2.049] \\
$\sigma=0$ & $\ell^p$ / CV spline & 2.535 & [1.869, 3.476] & [1.579, 4.192] \\
$\sigma=0.1$ & Banach / CV spline & 0.924 & [0.820, 1.037] & [0.767, 1.101] \\
$\sigma=0.1$ & $\ell^p$ / CV spline & 1.325 & [1.048, 1.682] & [0.917, 1.923] \\
$\sigma=0.3$ & Banach / CV spline & 1.010 & [0.917, 1.117] & [0.869, 1.184] \\
$\sigma=0.3$ & $\ell^p$ / CV spline & 1.054 & [0.950, 1.159] & [0.895, 1.218] \\
$\sigma=0.5$ & Banach / CV spline & 1.055 & [0.970, 1.163] & [0.935, 1.235] \\
$\sigma=0.5$ & $\ell^p$ / CV spline & 0.932 & [0.854, 1.019] & [0.816, 1.072] \\
$\sigma=1$ & Banach / CV spline & 0.915 & [0.813, 1.036] & [0.765, 1.110] \\
$\sigma=1$ & $\ell^p$ / CV spline & 0.767 & [0.680, 0.859] & [0.633, 0.910] \\
\multicolumn{5}{l}{\textbf{Sample scarcity: 18-equation core, $\sigma=0.3$}}\\*
$n=50$ & Banach / CV spline & 0.925 & [0.817, 1.040] & [0.758, 1.104] \\
$n=50$ & $\ell^p$ / CV spline & 0.862 & [0.764, 0.975] & [0.715, 1.043] \\
$n=100$ & Banach / CV spline & 0.977 & [0.866, 1.094] & [0.809, 1.160] \\
$n=100$ & $\ell^p$ / CV spline & 0.918 & [0.798, 1.042] & [0.728, 1.108] \\
$n=200$ & Banach / CV spline & 0.956 & [0.866, 1.055] & [0.820, 1.113] \\
$n=200$ & $\ell^p$ / CV spline & 0.883 & [0.806, 0.972] & [0.770, 1.027] \\
$n=500$ & Banach / CV spline & 1.010 & [0.917, 1.118] & [0.872, 1.182] \\
$n=500$ & $\ell^p$ / CV spline & 1.054 & [0.951, 1.161] & [0.890, 1.221] \\
$n=1000$ & Banach / CV spline & 0.897 & [0.814, 0.981] & [0.766, 1.027] \\
$n=1000$ & $\ell^p$ / CV spline & 1.022 & [0.893, 1.166] & [0.823, 1.254] \\
\multicolumn{5}{l}{\textbf{Exponent learning: 18-equation core, $n=500$}}\\*
$\sigma=0$ & $\ell^p$: learned / fixed $p=2.5$ & 0.467 & [0.356, 0.604] & [0.304, 0.688] \\
$\sigma=0.3$ & $\ell^p$: learned / fixed $p=2.5$ & 0.731 & [0.632, 0.838] & [0.578, 0.897] \\
\multicolumn{5}{l}{\textbf{Power-branch removal: 18-equation core, clean labels}}\\*
Clean & Banach / power branch off & 0.488 & [0.399, 0.592] & [0.354, 0.656] \\
\bottomrule\end{tabular}\end{table}
}

\section{Learned geometry}
The exponent summaries in Table~\ref{tab:geometry}(a) use the same clean 50-equation fits as the main comparison. All power edges start at $p=2.5$. Pooling edges gives medians of \LpExponentMedian{} for $\ell^p$ and \BanachExponentMedian{} for Banach; the corresponding percentiles and fractions below $p=2$ are reported in the table.

These distributions weight edges equally, so networks with more edges contribute more observations. They describe the shapes selected by this fitting procedure. They are not uncertainty intervals, and do not establish stability across different initialisations or a unique geometry for an equation. The comparison does not identify individual edge exponents as physical quantities.

\section{Extrapolation on all 50 targets}
\label{app:extrapolation}
\begin{table}[!htbp]
\centering\small
\caption{Extrapolation comparison of all 15 models on all 50 equations. Entries aggregate run medians within each equation. Wins and mean ranks compare all 15 models; lower NRMSE and rank and higher wins are better. Bold marks the best unrounded values.}
\label{tab:extrap_full}\begin{tabular}{lrrr}
\toprule
Model & NRMSE & Wins & Mean rank \\
\midrule
Banach & \textbf{0.416} & 5 & 5.98 \\
$\ell^p$ & 0.664 & 5 & 8.02 \\
Tanh & 1.031 & 0 & 8.96 \\
$\ell^p$ (additive $\varepsilon$) & 0.543 & 4 & 7.34 \\
Adaptive-slope Tanh & 1.205 & 1 & 9.72 \\
Activation mixture & 0.930 & 1 & 7.14 \\
Rational & 0.435 & 8 & \textbf{4.46} \\
Spline (Adam) & 0.582 & 5 & 5.12 \\
Spline (Adam, Lipschitz-reg.) & 1.476 & 0 & 11.44 \\
Chebyshev & 0.991 & 1 & 8.28 \\
RBF & 0.863 & 3 & 6.98 \\
Fourier & 1.669 & 0 & 11.84 \\
Wavelet & 1.146 & 0 & 11.00 \\
Jacobi & 1.045 & 2 & 8.54 \\
MLP & 0.691 & \textbf{15} & 5.18 \\
\bottomrule\end{tabular}

\end{table}
The extrapolation study uses all 50 targets, training-pool size 500 and 1,000 test examples. If an original coordinate interval is $[a,b]$, training draws occupy $[a+0.2(b-a),a+0.8(b-a)]$ and test draws occupy the upper band $[a+0.8(b-a),a+1.1(b-a)]$. The generator removes nonfinite target values and values with magnitude at least $10^6$ before selecting the required samples. Training labels are clean and use training-derived standardization. The models retain the matched Adam fitting and validation-selection rules.

Banach has the lowest aggregate median, while the rational control has the lowest mean rank and the MLP the most wins. Banach's geometric-mean paired ratios are \ExtrapSplineRatio{} against the spline and \ExtrapTanhRatio{} against $\tanh$; both family-wise intervals include one under the 26-comparison correction. The additive-$\varepsilon$ comparator uses the softplus exponent coordinates specified in Section~\ref{sec:implementation}; it is distinct from the C$^1$ response. All entries and paired intervals in this extrapolation comparison use the unified sample-SD population. One non-finite fit was replayed once locally with the unchanged protocol; its finite but extremely large error is retained, and the failed predecessor is preserved in the release audit. No successful replicate was discarded or selected by test error.

\section{Matched-parameter motivation experiment}
\label{app:figure1}
Figure~\ref{fig:fixed_bases} uses a step and a five-level staircase on $[-1,1]$. The fixed families are B-spline, Chebyshev, Gaussian radial basis, Wavelet (Haar) and Fourier. Each has 21 fitted coefficients. The geometric stack has five sets of $(a,w,b,\theta)$ plus one shared offset, also giving 21 fitted scalars. Fixed basis definitions and sampling locations are not trainable parameters.

Each model receives the same 2,000 equally spaced inputs and Gaussian label noise with standard deviation 0.03 in the original target units. Five noisy fits are retained. Fixed coefficients use linear least squares; the power stack uses the C$^1$ map with $\varepsilon=10^{-8}$ and the specified multi-restart L-BFGS fitter. It uses 18 restarts, at most 100 outer steps and 20 internal iterations per step. Restart selection uses noisy-training squared error only. Training compute is not matched.

The plotted curve uses seed 1729, and the shaded band is the pointwise range over fits, not a confidence interval. Main annotations are median noisy-training NRMSE. This illustration and the benchmark NRMSE use sample standard deviations. The additional shape diagnostics below use a fixed clean midpoint grid. Downward variation sums negative successive changes; overshoot measures departure outside the target range; refit root mean square (RMS) variation summarizes pointwise variation across fits. These are descriptive shape measurements, not robustness certificates.
\begin{table}[!htbp]
\centering\small
\caption{Step illustration: descriptive metrics over five noisy fits at 21 trainable parameters per model.}
\label{tab:figure1_step}\resizebox{\linewidth}{!}{\begin{tabular}{lrrrrrr}
\toprule
Model & Params & Noisy NRMSE & Clean NRMSE & Downward variation & Overshoot & Refit RMS \\
\midrule
B-spline & 21 & 0.1747 & 0.1651 & 0.4442 & 0.1044 & 0.0035 \\
Chebyshev & 21 & 0.1845 & 0.1763 & 0.8487 & 0.0929 & 0.0033 \\
Gaussian RBF & 21 & 0.1654 & 0.1546 & 0.5749 & 0.1011 & 0.0036 \\
Wavelet (Haar) & 21 & 0.0634 & 0.0067 & 0.0342 & 0.0060 & 0.0035 \\
Fourier & 21 & 0.2109 & 0.2011 & 1.9192 & 0.0939 & 0.0035 \\
$\ell^p$-KAN & 21 & 0.0594 & 0.0317 & 0.0210 & 0.0095 & 0.0128 \\
\bottomrule
\end{tabular}
}
\end{table}
\begin{table}[!htbp]
\centering\small
\caption{Staircase illustration: the same matched budget and descriptive metrics.}
\label{tab:figure1_staircase}\resizebox{\linewidth}{!}{\begin{tabular}{lrrrrrr}
\toprule
Model & Params & Noisy NRMSE & Clean NRMSE & Downward variation & Overshoot & Refit RMS \\
\midrule
B-spline & 21 & 0.1361 & 0.1069 & 0.2494 & 0.0290 & 0.0028 \\
Chebyshev & 21 & 0.1410 & 0.1131 & 0.3128 & 0.0198 & 0.0030 \\
Gaussian RBF & 21 & 0.1427 & 0.1148 & 0.3707 & 0.0305 & 0.0027 \\
Wavelet (Haar) & 21 & 0.1781 & 0.1543 & 0.0188 & 0.0058 & 0.0029 \\
Fourier & 21 & 0.2405 & 0.2243 & 1.5243 & 0.0969 & 0.0028 \\
$\ell^p$-KAN & 21 & 0.1450 & 0.1145 & 0.1342 & 0.0298 & 0.0322 \\
\bottomrule
\end{tabular}
}
\end{table}

The budget extension repeats the same noisy-training illustration at 13, 21, 41 and 125 trainable scalars per model, using the same retained noise draw for each target. The 21-parameter column reuses the corresponding Figure~\ref{fig:fixed_bases} fit. The remaining columns use the same basis definitions, C$^1$ response and fitting settings; only the number of coefficients or power edges changes. Table~\ref{tab:figure1_budget} reports noisy-training NRMSE, not an independent test benchmark. Nonlinear restarts are selected solely by noisy-training error. Equal parameter counts do not imply equal approximation families or fitting compute.
\begin{table}[!htbp]
\centering\small
\caption{Figure 1 budget sweep. Column headings give the exact trainable scalar count for every response. A single shared noise draw per target is used across budgets; the five-refit diagnostics above remain separate. Entries are noisy-training NRMSE with sample SD. This is a motivation study, not a superiority claim.}
\label{tab:figure1_budget}\begin{tabular}{llrrrr}
\toprule
Target & Response & 13 & 21 & 41 & 125 \\
\midrule
Step & B-spline & 0.229 & 0.173 & 0.126 & 0.087 \\
Step & Chebyshev & 0.233 & 0.185 & 0.136 & 0.092 \\
Step & Gaussian RBF & 0.208 & 0.163 & 0.123 & 0.086 \\
Step & Wavelet (Haar) & 0.064 & 0.063 & 0.063 & 0.062 \\
Step & Fourier & 0.267 & 0.211 & 0.156 & 0.103 \\
Step & $\ell^p$-KAN & 0.059 & 0.059 & 0.059 & 0.059 \\
\midrule
Staircase & B-spline & 0.185 & 0.133 & 0.108 & 0.090 \\
Staircase & Chebyshev & 0.191 & 0.139 & 0.118 & 0.093 \\
Staircase & Gaussian RBF & 0.141 & 0.139 & 0.112 & 0.090 \\
Staircase & Wavelet (Haar) & 0.205 & 0.175 & 0.132 & 0.103 \\
Staircase & Fourier & 0.298 & 0.240 & 0.180 & 0.124 \\
Staircase & $\ell^p$-KAN & 0.172 & 0.116 & 0.084 & 0.100 \\
\bottomrule\end{tabular}

\end{table}

\section{Implementation and theoretical qualifications}
\label{app:theoretical_scope}\label{app:theory}
For the smooth map, Eq.~\eqref{eq:derivative} is continuous at every fixed positive $\varepsilon$. Its positive bracket establishes convexity of the corresponding potential for $p>1$, and its finite maximum on a compact preactivation interval yields the bounded-domain edge sensitivity. For $1<p<2$, the origin slope can nevertheless be large. The additive-epsilon response has one-sided limits $\pm\varepsilon^{p-1}$ at zero; a numerical jump can be small without giving a Lipschitz map on an interval crossing zero.

The moment condition in Proposition~\ref{prop:moment} does not imply a bounded pointwise derivative: $3/2<p<2$ has finite moment but an unbounded derivative at zero. As smoothing tends to zero its near-origin contribution scales as $\varepsilon^{2p-3}$ for $1<p<3/2$, with logarithmic divergence at $p=3/2$. These are regularity statements rather than proofs of failure of a finite-smoothed optimiser.

A shared-argument response $c_1\tanh u+c_2J_p(u)$ is the derivative of $c_1\Psi_{\tanh}+c_2\Psi_p$, which is convex for nonnegative coefficients. Independent affine arguments and signed gains do not automatically preserve that interpretation. The smooth sensitivity calculations and exact-map proposition do not establish a whole-network robustness, optimisation or generalisation guarantee. Universality likewise does not explain the finite-budget ordering in the experiments.

We also plan to characterise the relevant KAN and MLP function classes through variation-norm spaces~\cite{bach2017breaking}, and to study more structured combinations of the $\tanh$ and power branches. The present results motivate these directions through specific operating regimes rather than a claim of universal superiority.

\section{Reproduction and experimental resources}
\label{app:reproduction}
The public kit at \url{https://github.com/seshkumar/geometric-kan-neurips2026} includes the model library, equation generators, fitting routines, full configurations and retained outputs. Its README specifies the Python environment, dependencies and commands. CPU execution is sufficient; the targets are generated without an external dataset download. The experiments use CPU training, and per-fit times are retained where recorded; they are not a controlled hardware-efficiency benchmark or a reconstruction of total research compute.

The in-domain studies use seeds 1729--1733. The separate extrapolation study uses seeds 0--4. Figure~\ref{fig:fixed_bases} uses data seeds 1729--1733 and a fixed optimiser-initialisation stream at 1729. The bootstrap seed is specified with its intervals above. These streams are distinct and their outputs are not pooled.

Run \texttt{python reproduce.py verify} to check the supplied release, and \texttt{python reproduce.py tables} to regenerate the numerical tables and motivation figure without training. For a fresh full-budget clean matched fit, use
\begin{quote}\small\ttfamily
python reproduce.py matched --condition baseline --eq-idx 0 --seed 1729 --model banach\_c1 --output fresh\_runs/banach.jsonl
\end{quote}
The corresponding \texttt{core} and \texttt{extrapolation} commands retain their own protocols. Fresh outputs do not replace the bundled results. The core pykan and matched Adam spline implementations are supplied separately. All model and equation configurations are defined by the supplied code.

Both in-domain runners standardize targets using training-pool statistics and evaluate against clean test targets. Adam checkpoint selection uses training-derived validation data only. Pykan uses a fixed fitting budget; its CV regularisation uses training-only folds. The matched evaluation denominator is the test-target sample standard deviation plus $10^{-12}$; the auxiliary core runner substitutes one for a nonfinite denominator or one below $10^{-8}$. These conventions are retained rather than silently harmonised. New fits may differ across hardware and numerical-library versions, especially with noisy or small training pools.

The kit supplies dependency versions and licence notices, including CPU-only PyTorch setup instructions. Original release code and generated outputs use Apache-2.0; dependencies retain their own terms. The release checks include imports, result-grid verification and regeneration from retained outputs. They are not a fresh environment installation or a repeat of the entire training population.

The empirical scope is symbolic and small tabular regression. The reported exponent distributions use a common initialisation, and sensitivity calculations do not imply adversarial robustness. Interpreting fitted exponents as physical quantities or transferring these benchmark conclusions to a consequential application would require additional evidence.

\subsection{PMLB tabular benchmark: additional notes}
\label{app:pmlb}
The per-dataset comparison in Table~\ref{tab:pmlb} retains the five-model main-text layout. Table~\ref{tab:pmlb_full} reports all 15 models on the same splits. Its two panels divide the columns for readability; each row's best value, wins and mean ranks are computed across all 15 models together. Dataset names in these tables are expanded in Table~\ref{tab:pmlb_datasets}.

Each of the 12 datasets uses five seeded 80/20 splits, with seeds 1729--1733. Features and targets are standardized using the training pool. The Adam trainer carves a 20\% validation subset from that pool, with validation seed zero, and selects the checkpoint using validation loss every 500 steps and at the last step. All models use 5,000 steps, initial learning rate $0.01$, cosine decay and gradient clipping at 10. The model builders, dimension-dependent widths and derivative penalty for the regularised spline are the same as in the main 15-model symbolic comparison. Both PMLB splines use Adam; the separate core pykan L-BFGS and cross-validation results are not part of this comparison.

PMLB NRMSE divides original-scale test RMSE by the sample standard deviation of the original test targets. Training-pool feature and target scaling also use sample standard deviations. A denominator below $10^{-8}$ is replaced by one. The SD convention therefore agrees with symbolic regression; the two benchmarks are not pooled. All reported fits completed their training budget with finite errors, and test labels are evaluated only after checkpoint selection.

\begin{table}[!htbp]
\centering\footnotesize\setlength{\tabcolsep}{3pt}
\caption{Complete PMLB comparison. Entries are per-dataset median test NRMSE over five runs. (a) Proposed and adaptive responses; (b) fixed-basis responses and MLP. Bold marks the best unrounded value across both panels; wins and ranks also compare all 15 models. The two spline models use $G=3$ and Adam.}
\label{tab:pmlb_full}
\par\smallskip\textbf{(a) Proposed and adaptive responses}\par\smallskip
\begin{tabular}{lrrrrrrr}
\toprule
Dataset & Banach & $\ell^p$ & Tanh & \shortstack{$\ell^p$\\(additive $\varepsilon$)} & \shortstack{Adaptive\\slope} & Mixture & Rational/Pad\'e \\
\midrule
Geyser & 0.470 & \textbf{0.449} & 0.460 & 0.463 & 0.475 & 0.468 & 0.468 \\
Environment & 1.042 & 0.934 & 0.919 & \textbf{0.842} & 0.880 & 1.101 & 1.044 \\
ESL & 0.538 & 0.407 & 0.373 & 0.383 & \textbf{0.358} & 0.379 & 0.401 \\
ERA & 0.808 & 0.817 & 0.808 & 0.813 & 0.810 & \textbf{0.807} & 0.810 \\
Cloud & 0.890 & 0.664 & 0.399 & 0.397 & \textbf{0.333} & 0.492 & 0.463 \\
Machine CPU & 0.826 & 0.555 & 0.293 & 0.464 & \textbf{0.276} & 0.334 & 0.734 \\
Sleuth case 2002 & 1.086 & 1.180 & 0.992 & 1.209 & \textbf{0.937} & 1.175 & 1.162 \\
CPU & \textbf{0.072} & 0.187 & 0.217 & 0.147 & 0.218 & 0.134 & 0.085 \\
PM10 & 0.947 & \textbf{0.870} & 0.914 & 0.942 & 0.977 & 0.965 & 0.919 \\
NO$_2$ & 0.831 & 0.726 & \textbf{0.685} & 0.728 & 0.731 & 0.761 & 0.772 \\
LA data & 0.831 & 0.643 & 0.627 & 0.635 & \textbf{0.616} & 0.803 & 0.695 \\
SWD & 0.816 & 0.791 & 0.794 & \textbf{0.787} & 0.794 & 0.826 & 0.812 \\
\midrule
Median NRMSE & 0.829 & 0.695 & \textbf{0.656} & 0.681 & 0.673 & 0.782 & 0.753 \\
Wins & 1 & 2 & 1 & 2 & \textbf{5} & 1 & 0 \\
Mean rank & 8.42 & 6.92 & \textbf{3.17} & 4.92 & 4.17 & 6.83 & 6.67 \\
\bottomrule\end{tabular}
\par\smallskip\textbf{(b) Fixed-basis responses and MLP}\par\smallskip
\begin{tabular}{lrrrrrrrr}
\toprule
Dataset & Spline & \shortstack{Spline\\Lipschitz-reg.} & Chebyshev & RBF & Fourier & Wavelet & Jacobi & MLP \\
\midrule
Geyser & 0.635 & 0.488 & 0.545 & 0.451 & 0.526 & 0.569 & 0.514 & 0.462 \\
Environment & 1.117 & 0.942 & 1.260 & 1.303 & 1.267 & 1.200 & 1.368 & 1.030 \\
ESL & 0.409 & 0.389 & 0.414 & 0.379 & 0.395 & 0.389 & 0.406 & 0.425 \\
ERA & 0.815 & 0.813 & 0.815 & 0.815 & 0.815 & 0.815 & 0.815 & 0.815 \\
Cloud & 0.574 & 0.484 & 0.396 & 0.424 & 0.493 & 0.735 & 0.596 & 0.401 \\
Machine CPU & 0.525 & 0.443 & 0.393 & 0.402 & 0.592 & 0.836 & 0.577 & 0.579 \\
Sleuth case 2002 & 1.548 & 1.100 & 1.030 & 1.112 & 1.112 & 1.216 & 1.357 & 1.169 \\
CPU & 0.136 & 0.176 & 0.118 & 0.098 & 0.419 & 0.508 & 0.216 & 0.124 \\
PM10 & 1.224 & 0.923 & 1.131 & 1.033 & 1.048 & 1.169 & 1.088 & 1.160 \\
NO$_2$ & 0.976 & 0.738 & 0.836 & 0.834 & 0.809 & 0.909 & 0.816 & 0.890 \\
LA data & 1.030 & 0.701 & 0.926 & 0.925 & 0.978 & 1.206 & 1.010 & 0.803 \\
SWD & 0.836 & 0.819 & 0.829 & 0.803 & 0.805 & 0.810 & 0.816 & 0.830 \\
\midrule
Median NRMSE & 0.825 & 0.719 & 0.822 & 0.809 & 0.807 & 0.826 & 0.815 & 0.809 \\
Wins & 0 & 0 & 0 & 0 & 0 & 0 & 0 & 0 \\
Mean rank & 12.25 & 6.92 & 9.00 & 7.08 & 10.50 & 12.58 & 11.17 & 9.42 \\
\bottomrule\end{tabular}

\end{table}

\begin{table}[!htbp]
\centering\footnotesize\setlength{\tabcolsep}{3pt}
\caption{PMLB datasets~\cite{romano2022pmlb}. The input dimension $d$ and total number of examples $n$ are checked against the retained fit records.}
\label{tab:pmlb_datasets}
\begin{tabular}{llrrl}
\toprule Dataset & Domain & $d$ & $n$ & Full PMLB identifier \\ \midrule
Geyser & geophysics & 2 & 222 & \texttt{712\_chscase\_geyser1} \\
Environment & environmental & 3 & 111 & \texttt{678\_visualizing\_environmental} \\
ESL & ordinal & 4 & 488 & \texttt{1027\_ESL} \\
ERA & ordinal & 4 & 1000 & \texttt{1030\_ERA} \\
Cloud & meteorology & 5 & 108 & \texttt{210\_cloud} \\
Machine CPU & computing & 6 & 209 & \texttt{230\_machine\_cpu} \\
Sleuth case 2002 & biology & 6 & 147 & \texttt{665\_sleuth\_case2002} \\
CPU & computing & 7 & 209 & \texttt{561\_cpu} \\
PM10 & pollution & 7 & 500 & \texttt{522\_pm10} \\
NO$_2$ & pollution & 7 & 500 & \texttt{547\_no2} \\
LA data & epidemiology & 10 & 508 & \texttt{666\_rmftsa\_ladata} \\
SWD & ordinal & 10 & 1000 & \texttt{1028\_SWD} \\
\bottomrule\end{tabular}

\end{table}

Tanh has the best mean rank among all 15 models; adaptive-slope tanh is also competitive. These outcomes do not isolate a benefit of the Banach interpretation on tabular data. The combined Banach response wins individual datasets but is weaker in aggregate than Tanh and the power-only response. The retained records do not include learned exponent distributions, so no tabular exponent interpretation is inferred.

Reproduce the fits with \texttt{experiments/neurips\_pmlb15.py run}, one declared dataset index at a time; \texttt{summarize} validates the complete population. The manuscript tables are generated by \texttt{python experiments/export\_neurips\_pmlb\_tables.py}. Run configurations and per-fit input identities accompany the retained outputs.

\section{Full benchmark (50 equations)}
\label{app:equations}
The table lists every target implemented by the benchmark generator, retaining the submitted columns for identifier, name, function type, dimension, formula and core membership. Let $d$ denote the input dimension and let $x_1,\ldots,x_d$ denote the coordinates in generator order. Rows 1--18 form the core subset used for the auxiliary controls. All 50 rows are used for the main clean, noise and scarcity comparisons and for extrapolation. The sampling domains are specified by the supplied generator and retained catalogue.

The function types are polynomial (poly), rational (rat), power (pow), trigonometric (trig), exponential (exp), logarithmic (log), hyperbolic (hyp) and compositional (comp). The identifiers follow the benchmark's AI Feynman-derived naming~\cite{udrescu2020ai}; entries prefixed \texttt{F.} are synthetic stress tests. Names are descriptive; PE abbreviates potential energy. The coordinate formulas below specify the implemented tasks, without requiring an additional interpretation of the symbols as physical units. Training and evaluation use the sampling and validity rules given above.
\begin{table}[!htbp]\centering\scriptsize\setlength{\tabcolsep}{4pt}
\caption{All 50 benchmark equations. Names and identifiers follow the implemented benchmark; entries prefixed \texttt{F.} are synthetic stress tests. ``Core'' marks membership in the 18-equation subset.}\label{tab:equations}
\begin{tabular}{clllcl@{}c}
\toprule \# & ID & Name & Type & $d$ & Formula & Core \\ \midrule
1 & I.12.1 & Coulomb & poly & 2 & $x_{1} x_{2}$ & $\checkmark$ \\
2 & I.14.4 & Spring energy & poly & 2 & $x_{1} x_{2}^{2} / 2$ & $\checkmark$ \\
3 & I.25.13 & Capacitor & rat & 2 & $x_{1} / x_{2}$ & $\checkmark$ \\
4 & I.34.14 & Rel. Doppler & pow & 3 & $x_{1} \sqrt{\tfrac{x_{2} / x_{3} + 1}{- x_{2} / x_{3} + 1}}$ & $\checkmark$ \\
5 & I.14.3 & Grav. PE & poly & 3 & $x_{1} x_{2} x_{3}$ & $\checkmark$ \\
6 & I.18.12 & Torque & trig & 3 & $x_{1} x_{2} \sin{\left(x_{3} \right)}$ & $\checkmark$ \\
7 & I.34.10 & Doppler & rat & 3 & $\tfrac{x_{3}}{1 - x_{2} / x_{1}}$ & $\checkmark$ \\
8 & I.47.23 & Speed of sound & pow & 3 & $\sqrt{x_{1} x_{2} / x_{3}}$ & $\checkmark$ \\
9 & III.17.37 & Ang. distribution & trig & 3 & $x_{1} \left(x_{2} \cos{\left(x_{3} \right)} + 1\right)$ & $\checkmark$ \\
10 & I.8.14 & Euclidean dist. & comp & 4 & $\sqrt{\left(- x_{1} + x_{2}\right)^{2} + \left(- x_{3} + x_{4}\right)^{2}}$ & $\checkmark$ \\
11 & I.13.4 & Kinetic energy & poly & 4 & $x_{1} \left(x_{2}^{2} + x_{3}^{2} + x_{4}^{2}\right) / 2$ & $\checkmark$ \\
12 & I.29.16 & Law of cosines & comp & 4 & $\sqrt{x_{1}^{2} - 2 x_{1} x_{2} \cos{\left(x_{3} - x_{4} \right)} + x_{2}^{2}}$ & $\checkmark$ \\
13 & III.4.32 & Bose-Einstein & exp & 4 & $\tfrac{1}{e^{\tfrac{x_{1} x_{2}}{2 \pi x_{3} x_{4}}} - 1}$ & $\checkmark$ \\
14 & I.12.11 & Lorentz force & trig & 5 & $x_{1} \left(x_{2} + x_{3} x_{4} \sin{\left(x_{5} \right)}\right)$ & $\checkmark$ \\
15 & I.44.4 & Entropy change & log & 5 & $x_{1} x_{2} x_{3} \log{\left(x_{5} / x_{4} \right)}$ & $\checkmark$ \\
16 & II.35.21 & Magnetisation & hyp & 5 & $x_{1} x_{2} \tanh{\left(x_{2} x_{3} / x_{4} x_{5} \right)}$ & $\checkmark$ \\
17 & III.14.14 & Diode equation & exp & 5 & $x_{1} \left(e^{x_{2} x_{3} / x_{4} x_{5}} - 1\right)$ & $\checkmark$ \\
18 & I.40.1 & Barometric & exp & 6 & $x_{1} e^{- x_{2} x_{3} x_{5} / x_{4} x_{6}}$ & $\checkmark$ \\
19 & II.6.15a & Dipole field & trig & 3 & $\tfrac{x_{1} \cos{\left(x_{2} \right)}}{x_{3}^{2}}$ &  \\
20 & I.12.4 & Electric field & rat & 2 & $\tfrac{x_{1}}{x_{2}^{2}}$ &  \\
21 & I.10.7 & Relativistic mass & pow & 2 & $\tfrac{x_{1}}{\sqrt{1 - x_{2}^{2}}}$ &  \\
22 & I.6.20a & Normal distribution & exp & 2 & $\tfrac{\sqrt{2} e^{- \tfrac{x_{1}^{2}}{2 x_{2}^{2}}}}{2 \sqrt{\pi} x_{2}}$ &  \\
23 & F.43 & Velocity correction & poly & 3 & $x_{1} x_{2} x_{3} + x_{1} + x_{2}$ &  \\
24 & I.16.6 & Velocity addition & rat & 3 & $\tfrac{x_{1} + x_{2}}{\tfrac{x_{1} x_{2}}{x_{3}^{2}} + 1}$ &  \\
25 & I.27.6 & Thin lens & rat & 3 & $\tfrac{1}{x_{2} / x_{3} + 1 / x_{1}}$ &  \\
26 & II.2.42 & Heat conduction & rat & 3 & $x_{1} x_{2} / x_{3}$ &  \\
27 & II.15.4 & Magnetic PE & trig & 3 & $x_{1} x_{2} \cos{\left(x_{3} \right)}$ &  \\
28 & III.15.12 & Tight binding & trig & 3 & $2 x_{1} \left(1 - \cos{\left(x_{2} x_{3} \right)}\right)$ &  \\
29 & I.18.16 & Angular momentum & trig & 4 & $x_{1} x_{2} x_{3} \sin{\left(x_{4} \right)}$ &  \\
30 & II.11.17 & Density fluctuation & trig & 4 & $x_{1} \left(x_{2} \cos{\left(x_{3} \right)} / x_{4} + 1\right)$ &  \\
31 & I.37.4 & Interference & comp & 3 & $x_{1} + x_{2} + 2 \sqrt{x_{1} x_{2}} \cos{\left(x_{3} \right)}$ &  \\
32 & I.26.2 & Snell's law & trig & 2 & $\operatorname{asin}{\left(x_{1} \sin{\left(x_{2} \right)} \right)}$ &  \\
33 & F.33 & Mixed trig product & trig & 5 & $x_{1} x_{2} x_{3} \sin{\left(x_{4} \right)} \cos{\left(x_{5} \right)}$ &  \\
34 & F.34 & Phase-shift product & trig & 5 & $x_{1} x_{2} x_{5} \sin{\left(x_{3} + x_{4} \right)}$ &  \\
35 & III.10.19 & Magnetic moment norm & comp & 4 & $x_{1} \sqrt{x_{2}^{2} + x_{3}^{2} + x_{4}^{2}}$ &  \\
36 & I.15.3t & Lorentz time & comp & 4 & $\tfrac{x_{1} - \tfrac{x_{2} x_{3}}{x_{4}^{2}}}{\sqrt{- \tfrac{x_{2}^{2}}{x_{4}^{2}} + 1}}$ &  \\
37 & I.24.6 & Oscillator energy & poly & 4 & $x_{1} x_{4}^{2} \left(x_{2}^{2} + x_{3}^{2}\right) / 4$ &  \\
38 & I.18.4 & Center of mass & rat & 4 & $\tfrac{x_{1} x_{2} + x_{3} x_{4}}{x_{1} + x_{3}}$ &  \\
39 & I.13.12 & Grav. PE diff & rat & 5 & $x_{1} x_{2} x_{3} \left(- 1 / x_{5} + 1 / x_{4}\right)$ &  \\
40 & F.40 & Inverse-sqrt field & comp & 4 & $\tfrac{x_{1} x_{2}}{\sqrt{x_{3}^{2} + x_{4}^{2}}}$ &  \\
41 & I.39.11 & Polytropic energy & rat & 3 & $\tfrac{x_{1} x_{2}}{x_{3} - 1}$ &  \\
42 & I.48.20 & Relativistic energy & pow & 3 & $\tfrac{x_{1} x_{2}^{2}}{\sqrt{1 - \tfrac{x_{3}^{2}}{x_{2}^{2}}}}$ &  \\
43 & II.11.27 & Boltzmann factor & exp & 4 & $x_{1} e^{- x_{2} x_{3} / x_{4}}$ &  \\
44 & F.44 & Fermi-Dirac & exp & 4 & $\tfrac{1}{e^{x_{1} x_{2} / x_{3} x_{4}} + 1}$ &  \\
45 & F.45 & Saturating exponential & exp & 3 & $x_{1} \left(1 - e^{- x_{2} / x_{3}}\right)$ &  \\
46 & F.46 & Saturating response & hyp & 5 & $x_{1} x_{2} \tanh{\left(x_{3} / x_{4} x_{5} \right)}$ &  \\
47 & F.31 & Polarization response & rat & 5 & $\tfrac{x_{1} x_{2}}{x_{3} \left(x_{4}^{2} - x_{5}^{2}\right)}$ &  \\
48 & F.48 & Nested logarithm & log & 5 & $x_{1} \log{\left(x_{2} x_{3} / x_{4} x_{5} + 1 \right)}$ &  \\
49 & F.49 & 6D multiplicative exp & exp & 6 & $x_{1} x_{2} e^{- x_{3} x_{4} / x_{5} x_{6}}$ &  \\
50 & III.21.1 & Current density & poly & 5 & $x_{1} x_{2} x_{3} x_{4} / x_{5}$ &  \\
\bottomrule\end{tabular}\end{table}

\end{document}